\documentclass[11pt]{article}
\usepackage[margin=1in]{geometry}
\usepackage{amsmath,amssymb,amsthm}
\usepackage{booktabs}
\usepackage{graphicx}
\usepackage{microtype}
\usepackage{placeins}
\usepackage{longtable}
\usepackage{array}
\usepackage[numbers,sort&compress]{natbib}
\usepackage[hidelinks]{hyperref}

\usepackage{tikz}
\usetikzlibrary{positioning, fit, arrows.meta}

\theoremstyle{definition}
\newtheorem{theorem}{Theorem}
\newtheorem{proposition}[theorem]{Proposition}
\newtheorem{corollary}[theorem]{Corollary}
\newtheorem{definition}[theorem]{Definition}
\newtheorem{lemma}[theorem]{Lemma}
\theoremstyle{remark}
\newtheorem{remark}[theorem]{Remark}

\newcommand{\qQQ}{q_{\mathrm{QQ}}}
\newcommand{\OSS}{\mathrm{OSS}}
\newcommand{\codeurl}{\url{https://github.com/pilsungk/llm-qq}} 

\title{Auditing Question-Order Effects in Large Language Models with the QQ Equality: Mechanism Characterization and a Saturation Caveat}
\author{Pilsung Kang\\
Department of Software Science, Dankook University\\
Yongin, South Korea\\
\texttt{pilsungk@dankook.ac.kr}}
\date{}

\begin{document}
\maketitle

\begin{abstract}
Question-order effects in human survey data have been reported to approximately
satisfy the QQ (quantum question) equality, a parameter-free prediction of the
standard projective quantum question-order model. We develop this equality into
an audit framework for sequential binary judgments of autoregressive large
language models (LLMs). Theoretically, we characterize mechanism families that
satisfy QQ robustly, show that classical repetition can reproduce the equality
exactly, and combine QQ with the rank-2 Contextuality-by-Default criterion
through $|\qQQ|\le\OSS$. This separates order sensitivity, QQ imbalance, and
residual contextuality rather than treating them as interchangeable signatures.
Methodologically, we introduce a committed multi-turn forced-branch protocol
that reconstructs order-conditioned joint distributions from next-token
log-probabilities under counterbalanced label mappings and pre-specified health
gates.  A first-signal pilot on an open-weight instruction-tuned model reveals
the central measurement problem. Although all pre-specified health gates passed,
the binary-conditioned distributions were near-deterministic for 17 of 18 item
pairs under the direct-evaluation framing and 7 of 8 under the persona framing.
Label assignment materially changed several mapping-specific QQ verdicts, and no
item was certified as residually contextual.  Thus, under the tested conditions,
the observed QQ outcomes did not uniquely identify a response mechanism in the
presence of a saturated and label-sensitive measurement interface. The main
implication is methodological: next-token probabilities should not be
interpreted as survey-response distributions without first establishing adequate
dispersion. We therefore argue that saturation screening and label
counterbalancing should precede structural interpretation in distribution-level
audits of LLM judgments.
\end{abstract}

\section{Introduction}

Question order changes human survey answers, sometimes dramatically.
In a Gallup poll conducted in September 1997, half of the respondents were asked about Bill Clinton first and then
Al Gore; the other half answered in the opposite order. In the
Clinton-first group, 50\% rated Clinton honest and trustworthy, and
Gore, asked immediately after, received 60\%. In the Gore-first
group, 68\% rated Gore honest and trustworthy, while Clinton, asked
second, received 57\%. Each politician's rating thus moved with
question position --- Clinton from 50\% to 57\%, Gore from 68\% to
60\% --- so that the 18-point gap between their first-position
ratings shrank to 3 points in second position \citep{moore:2002,wang:2013:qq}.
Order effects of this kind are pervasive
and well documented, and by themselves they are not surprising.

What is surprising is a sharper regularity hiding beneath them. Define
the order-asymmetry statistic
\begin{equation}
\qQQ := \bigl[p_{AB}(A_y,B_n)+p_{AB}(A_n,B_y)\bigr]
      - \bigl[p_{BA}(B_y,A_n)+p_{BA}(B_n,A_y)\bigr],
\end{equation}
where $p_{AB}$ and $p_{BA}$ denote the joint answer distributions when
question $A$ is asked first and when question $B$ is asked first,
respectively, and $A_y$, $A_n$, $B_y$, $B_n$ denote ``yes'' and ``no''
answers to each question; 
$\qQQ$ is the difference between the two orders in the probability
that the two answers disagree, and probability theory alone does not
force it to vanish, since $p_{AB}$ and $p_{BA}$ come from different
respondent groups.
The \emph{QQ (quantum question)
equality} is the statement $\qQQ = 0$. \citet{wang:2014:pnas}
documented that across 72 datasets, mostly national field studies
(plus laboratory experiments), this equality holds to close
approximation.
This is an a priori, parameter-free prediction of the projective
quantum question-order model~\citep{wang:2013:qq}. Equivalently, the
probability that the two answers \emph{disagree} is the same in both
orders: in the Clinton--Gore poll above, the disagreement probability
was $0.221$ in one order and $0.225$ in the other, even though the
individual agreement rates moved by up to eight points. The equality
is thus a non-parametric constraint on the pair of joint distributions
--- not a no-order-effect claim --- and its empirical success is what
makes it a candidate audit criterion rather than a curiosity.

Large language models (LLMs) are increasingly used as proxies for human respondents, and exhibit well-documented order and format sensitivities~\citep{tjuatja:2024:acl,rupprecht:2026:acl,sandhan:2025:emnlp}. This raises a natural audit question: does the sequential judgment behavior of an autoregressive LLM --- one that defines a probability distribution over the next token given the preceding text --- satisfy the same non-parametric constraint that human populations do? To the best of our knowledge (within our 2026-07 search scope; Section~\ref{sec:related}), no prior work has tested the QQ equality on LLMs under a multi-turn sequential protocol with answers committed to the conversation history.

A human study estimates response distributions from repeated sampling across
respondents. By contrast, an LLM study typically treats the model's next-token
probability distribution, read out at a forced-choice answer position, as the
response distribution of a synthetic respondent. This substitution is convenient
and widely used, but it rests on a measurement assumption rather than a fact:
that the binary-conditioned pair obtained by renormalizing the two label
probabilities can genuinely play the role that a response distribution plays in
a human study. After this conditioning, the pair sums to one by construction,
but it need not carry usable distributional structure; this is a real
possibility for an instruction-tuned model --- one further trained to follow
instructions and answer as an assistant, as is typical of deployed chat models ---
facing a forced binary choice. The QQ equality is informative as an audit
criterion only when the assumption holds, and making it explicit and testable is
as much a part of our task as the QQ test itself.

One caution governs the interpretation of any outcome. We deliberately use compatibility language throughout: QQ satisfaction is consistent with projective measurement \emph{and} with symmetric classical repetition (Proposition~\ref{prop:symrep}); QQ violation rejects only the standard projective model, not ``quantum-like'' accounts in general --- quantum instruments based on positive operator-valued measures (POVMs) can violate QQ~\citep{lebedev:2023}.

Answering this question well therefore requires more than running a test. It requires knowing \emph{what a result would mean} (which mechanism classes are compatible with satisfaction or violation), \emph{whether the measurement instrument is adequate}, and \emph{what an inadequate instrument looks like in practice}. Our contributions track exactly these three needs:
\begin{enumerate}
\item \textbf{Mechanism characterization (Section~\ref{sec:theory}).} We
formalize mechanisms as maps from item parameters to sequential response
kernels and characterize robust QQ satisfaction in two families
(Theorem~\ref{thm:mi}, Proposition~\ref{prop:mixture}), prove closure
under order-matched mixing (Proposition~\ref{prop:mixing}), exhibit a
classical mechanism that satisfies QQ identically while producing order
effects (Proposition~\ref{prop:symrep}), and translate the rank-2 Contextuality-by-Default (CbD)
criterion into audit coordinates (Corollary~\ref{cor:cbd}). 
The result is a two-layer discriminant table that separates mechanism compatibility from CbD certification.
\item \textbf{Audit methodology (Section~\ref{sec:method}).} We
develop a pre-specified, audit-logged pipeline for the
\emph{binary-conditioned audit estimand} $P(y\mid y\text{ or }n)$
under multi-turn sequential presentation. The pipeline applies to
any model that exposes next-token log-probabilities and admits
single-token answer labels. It provides exact component envelopes
for unassigned probability mass, sound and potentially conservative
certification bounds, pre-specified Clopper--Pearson spot
checks~\citep{clopper:1934}, full label counterbalancing with
semantic canonicalization, and a saturation diagnostic.
\item \textbf{A first-signal pilot and a measurement caveat
(Section~\ref{sec:pilot}).} We report a first-signal pilot on
Qwen3-4B-Instruct-2507~\citep{yang:2025:qwen3} under two framings,
direct evaluation and survey-respondent persona. All pre-specified
health checks passed, yet the response distributions were saturated
(near-deterministic) for 17/18 and 7/8 item pairs respectively.
Observed QQ violations reflected saturated-regime patterns,
deterministic order flips and label-mixture effects, rather than a
single response mechanism. No item was certified residually
contextual, a few remained indeterminate at the certification
bounds, and label assignment had a substantive effect. The central
finding is that forced-binary next-token log-probabilities failed to
expose the distributional structure the QQ audit needs \emph{under
the tested model and prompting conditions}. A pre-specified
saturation diagnostic catches this failure early and auditably.
\end{enumerate}

Taken together, these contributions shift the audit question from
whether an LLM satisfies the QQ equality to whether the measurement
interface is sufficiently informative for a QQ verdict to support
structural interpretation. Measurement adequacy is therefore not
merely an implementation detail, but a prerequisite for
claims about the underlying response process.

\section{Related Work}\label{sec:related}

Our work sits at the intersection of two literatures: the quantum-cognition line that produced the QQ equality and its contextuality-theoretic refinements, and the language-model evaluation line that treats LLMs as survey respondents and scrutinizes how their answers are measured.

\paragraph{QQ equality and quantum question-order models.}
\citet{wang:2013:qq} derived the QQ equality from projective sequential measurement and showed that standard Bayesian and Markov accounts do not generally satisfy it; \citet{wang:2014:pnas} confirmed it across 72 datasets. The Supporting Information of \citet{wang:2014:pnas} further shows that a repeat-choice model constrained to satisfy QQ predicts context effects proportional to marginal differences --- empirically rejected in humans ($r=.04$) --- and that an anchoring-adjustment model satisfies QQ iff its two order biases are equal; both facts foreshadow our characterization results, which we state and prove at the mechanism level. \citet{ozawa:2021:jmp} unify order effects, response replicability, and the QQ equality with quantum instruments; \citet{lebedev:2023} show general POVMs can violate QQ.

\paragraph{Contextuality-by-Default.} For inconsistently connected systems, CbD~\citep{dzhafarov:2016:cbd,kujala:2016} provides the correct contextuality criterion; for rank-2 cyclic systems it reduces to a single inequality that we show equals $|\qQQ|\le\OSS$ in audit coordinates (Corollary~\ref{cor:cbd}). Prior CbD analyses of question-order data interpret human order effects as direct influences.  Recent concurrent work applies CbD to generative LLMs in other task families: \citet{kumar:2026:fail-invariance-llm} report irreducible context-dependence in a pronoun-inference setting, while
\citet{tuan:2026} reports zero corrected CbD degree in an
enterprise-agent auditing pilot, with direct influence reported as the
dominant signal. The latter result is qualitatively aligned with the
absence of certified residual contextuality in our pilot, although our
audit does not identify direct influence as the generating mechanism.

\paragraph{Quantum-inspired tests on LLMs.} \citet{imannezhad:2026} test the conjunction fallacy, the disjunction fallacy, and binary complementarity on GPT-5; \citet{aerts:2026:entropy} report Bell-type violations and Bose--Einstein fits in concept combinations (interpretation debated); \citet{lo:2025} report large-scale sheaf- and CbD-contextuality in BERT probability distributions over a co-reference-style linguistic schema. None tests the QQ equality under sequential multi-turn presentation.

\paragraph{LLMs as survey respondents.} \citet{tjuatja:2024:acl} document human-like response biases; 
selection-order and option-label biases in multiple-choice formats are likewise documented in single-shot
settings~\citep{zheng:2024:llm-mcq-bias,pezeshkpour:2024:llm-mcq-sens}, which our counterbalanced label mappings are designed to control and quantify in the sequential regime; \citet{rupprecht:2026:acl} test 18 LLMs on World Values Survey items under ten perturbations (334{,}800 interviews), finding a consistent recency bias in almost all tested models; \citet{sandhan:2025:emnlp} evaluate LLM personality tests with conversational context and report model-dependent sensitivity to question ordering. Our results add a distribution-level caveat to this literature: under forced-binary formats the ``response distribution'' of an instruction-tuned model can be an artifact of a saturated next-token probability distribution, so distribution-level constructs (including QQ) require measurement designs that restore genuine dispersion.

\paragraph{First-token measurement critique.} \citet{wang:2024} show first-token probabilities can disagree with generated text answers. Our saturation finding extends the critique from \emph{disagreement} to \emph{degeneracy}: even when first-token measurement is internally consistent (our logprob and trajectory-sampling pathways agree, in the pre-specified spot checks, within exact Clopper--Pearson intervals), the measured distribution may carry almost no information beyond its argmax.

\section{Theory: Which Mechanisms Satisfy the QQ Equality?}
\label{sec:theory}

This section develops the interpretive machinery for reading a QQ
verdict. Two asymmetries motivate it: QQ violation alone does not
establish residual contextuality, since order effects themselves can
account for it~\citep{dzhafarov:2016:cbd}, and, as we show, QQ
satisfaction alone does not identify a unique generating mechanism.
We first characterize the mechanism classes compatible with
satisfaction, then translate the rank-2 CbD criterion into audit
coordinates, and combine the two into a two-layer discriminant table. Its
certification layer underpins the empirical analysis of
Section~\ref{sec:pilot}, and its compatibility layer states
theoretical diagnostic signatures.

\subsection{Setup}
\label{ss:setup}

For binary questions $A,B$ with answers $\alpha,\beta\in\{y,n\}$,
define first-position distributions $p_A(\alpha)$, $p_B(\beta)$ and
order-conditioned kernels $K_{AB}(\beta\,|\,\alpha)$,
$K_{BA}(\alpha\,|\,\beta)$ --- the conditional distributions of the
second answer given the first, one per order --- with joint distributions
$p_{AB}(\alpha,\beta)=p_A(\alpha)K_{AB}(\beta\,|\,\alpha)$ and
$p_{BA}(\beta,\alpha)=p_B(\beta)K_{BA}(\alpha\,|\,\beta)$, with joint
arguments listed in the order the questions are asked.
Writing $M_{XY}$ for the probability that the two answers disagree in order $XY$, $\qQQ = M_{AB}-M_{BA}$.

\begin{definition}[Mechanism; robust QQ]\label{def:mechanism}
A \emph{mechanism} is a map from \emph{item parameters} --- quantities
that vary across items, minimally the first-position tendencies
$(p_A,p_B)$ --- to the pair of order-conditioned kernels, with
\emph{internal parameters} (quantities fixed across items, such as a
repetition rate) held constant. It satisfies QQ \emph{robustly} if
$\qQQ=0$ throughout an admissible domain of item parameters whose
first-position marginals contain an open subset of $(0,1)^2$ ---
that is, the equality holds structurally rather than at fine-tuned
parameter points.
\end{definition}

Finite experiments certify compatibility only, never robustness. 
We write the \emph{mismatch transition rates} --- the probabilities
that the second answer differs from the first --- as
$a=K_{AB}(n|y)$, $b=K_{AB}(y|n)$, $c=K_{BA}(n|y)$, and $d=K_{BA}(y|n)$.
The \emph{order effects} on the marginals are
$OE_A = p_{BA}(A_y)-p_{AB}(A_y)$ and $OE_B = p_{AB}(B_y)-p_{BA}(B_y)$:
the change in each question's ``yes'' marginal when it is asked second
rather than first. We abbreviate $q$ for $\qQQ$ where unambiguous.

Throughout the paper, results are stated and proved in classical probability terms (response kernels and transition rates); quantum models enter only as particular mechanism classes to be included in or excluded from compatibility statements (e.g., Theorem~\ref{thm:wb}), and no quantum formalism is needed to follow the audit methodology or the empirical results.

\subsection{Mechanism characterization}
\label{ss:mechanism}

We now characterize which mechanisms satisfy the QQ equality robustly.
Proofs of all results in this subsection are given in Appendix~\ref{app:proofs}.

\begin{theorem}[Wang--Busemeyer \citep{wang:2013:qq}]
\label{thm:wb}
Projective sequential measurement satisfies $\qQQ=0$ for every state and
every pair of binary projective measurements.
\end{theorem}
We reproduce this numerically to machine precision in our validation
suite. In the language of Section~\ref{ss:setup}, the projective
model is a mechanism whose item parameters are the state and
measurement axes, and the universal quantifier makes its QQ
satisfaction robust in the sense of Definition~\ref{def:mechanism}.
The next results show that this robust-satisfaction class also
contains purely classical mechanisms.

\begin{theorem}[Marginal-independent kernels]\label{thm:mi}
If the mismatch transition rates $a=K_{AB}(n|y)$, $b=K_{AB}(y|n)$,
$c=K_{BA}(n|y)$, $d=K_{BA}(y|n)$ are constants independent of
$(p_A,p_B)$, then with $u=p_A(y)$, $v=p_B(y)$,
\begin{equation}
\qQQ = u(a-b) - v(c-d) + (b-d),
\end{equation}
and robust QQ holds \textbf{iff} $a=b=c=d$.
\end{theorem}
The 2D rank-1 projective model \textbf{with a fixed pair of measurements
under state variation} lies in this class, with
$a=b=c=d=\mathrm{Tr}(P_{A,y}P_{B,n})=|\langle A_y|B_n\rangle|^2$; if
projectors vary across items, the fixed-kernel hypothesis of
Theorem~\ref{thm:mi} no longer applies.

\begin{proposition}[Symmetric repetition]\label{prop:symrep}
Fix $r\in[0,1)$ and let
$K_{AB}(\beta\,|\,\alpha)=r\,\mathbf{1}[\beta=\alpha]+(1-r)\,p_B(\beta)$
and, symmetrically,
$K_{BA}(\alpha\,|\,\beta)=r\,\mathbf{1}[\alpha=\beta]+(1-r)\,p_A(\alpha)$.
Then:
\begin{enumerate}
\item[(i)] $\qQQ=0$ identically in $(p_A,p_B)$, i.e., QQ holds
robustly, while $OE_B = r\,(p_A(y)-p_B(y))$;
\item[(ii)] the mismatch rates $a=(1-r)(1-p_B(y))$ and
$b=(1-r)\,p_B(y)$ differ whenever $p_B(y)\neq 1/2$, whereas any 2D
rank-1 projective model has $a=b$ (Lemma~\ref{lem:transym}, Appendix~\ref{app:proofs}).
\end{enumerate}
\end{proposition}

Hence, at the kernel level, symmetric repetition is not realizable by
any 2D rank-1 projective model at any item with $p_B(y)\neq 1/2$;
identifying both kernel branches from observed behavior additionally
requires interior first-answer marginals. QQ satisfaction therefore
does not uniquely identify a (2D rank-1) projective mechanism.

\begin{proposition}[Mixture-repetition family]\label{prop:mixture}
Fix polarity- and order-dependent retention rates
$r_{AB}[\alpha]$, $r_{BA}[\beta]\in[0,1]$ --- where $r_{AB}[\alpha]$
denotes the retention rate after first answer $\alpha$ ($y$ or $n$)
in order $AB$, and symmetrically for $r_{BA}[\beta]$ --- and let
$K_{AB}(\beta\,|\,\alpha)
 = r_{AB}[\alpha]\,\mathbf{1}[\beta=\alpha]
 + (1-r_{AB}[\alpha])\,p_B(\beta)$,
and symmetrically for $K_{BA}$. Then:
\begin{enumerate}
\item[(i)]
$\qQQ = p_A(y)p_B(n)\,(r_{BA}[n]-r_{AB}[y])
      + p_A(n)p_B(y)\,(r_{BA}[y]-r_{AB}[n])$,
and robust QQ holds iff $r_{AB}[y]=r_{BA}[n]$ and
$r_{AB}[n]=r_{BA}[y]$ (cross-polarity, cross-order symmetry);
\item[(ii)] order-invariant retention
($r_{AB}[\cdot]=r_{BA}[\cdot]=r[\cdot]$) gives
$q=(r[n]-r[y])\,(p_A(y)-p_B(y))$;
\item[(iii)] polarity-invariant retention
($r_{XY}[y]=r_{XY}[n]=r_{XY}$) gives
$q=(r_{BA}-r_{AB})\,S$, where $S := p_A(y)p_B(n)+p_A(n)p_B(y)$.
\end{enumerate}
\end{proposition}

The two closed forms (ii)--(iii) provide theoretically motivated
diagnostic signatures for observed violations: polarity-dependent
repetition (acquiescence-type asymmetry) and order-dependent
repetition, respectively. 

\begin{proposition}[Order-matched mixing closure]\label{prop:mixing}
Let $J^{(1)},\dots,J^{(K)}$ be joint pairs each satisfying
$\qQQ(J^{(k)})=0$, and let $w_1,\dots,w_K\ge 0$, $\sum_k w_k = 1$, be
mixture weights applied identically to both orders (the weights may
depend on the item parameters $\theta$ but not on the order). Then:
\begin{enumerate}
\item[(i)] the mixture $\sum_k w_k J^{(k)}$ satisfies $\qQQ=0$;
\item[(ii)] if the weights may instead depend on the order, the
mixture can violate QQ: components with different order-invariant
disagreement probabilities, selected with order-dependent weights,
yield $\qQQ\neq 0$.
\end{enumerate}
\end{proposition}

Order-matched mixing thus provides one sufficient route by which
heterogeneous populations may preserve QQ satisfaction in aggregate;
it is consistent with the pooled estimand used for counterbalanced
label mappings (Section~\ref{sec:method}); and it marks
order-dependent routing (e.g., in mixture-of-experts decoding) as an
audit checkpoint --- a possible violation channel, not a necessity.

We now translate the CbD noncontextuality criterion into audit
coordinates. The relevant allowance is the total order sensitivity of
the marginals: define the \emph{order-sensitivity score} of an item
as $\OSS := |OE_A| + |OE_B|$.

\begin{corollary}[CbD criterion in audit coordinates]\label{cor:cbd}
For the rank-2 cyclic system with $\pm1$ coding, the
Kujala--Dzhafarov noncontextuality criterion \citep{kujala:2016} is
equivalent to
\begin{equation}
|\qQQ| \;\le\; \OSS .
\end{equation}
\end{corollary}

The mathematical content of the criterion is due
to~\citet{kujala:2016}; the contribution here is the translation into
audit coordinates (Appendix~\ref{app:proofs}). The inequality
supports a three-way reading: $\qQQ=0$ implies noncontextuality;
violations with $0<|q|\le\OSS$ are compatible with a noncontextual
coupling that allows direct influences --- the criterion does not
identify direct influence as the generating cause; and only
$|q|>\OSS$ certifies residual contextuality. $\OSS$ thereby functions
as an inconsistent-connectedness allowance.

The propositions above characterize specific mechanism families;
a general characterization of robust QQ satisfaction remains open.
Beyond the projective class, unsharp (POVM-based) measurement models
need not satisfy QQ~\citep{lebedev:2023}. Exploratory numerical
sweeps over qubit L\"uders instruments~\citep{luders:1950}, 
included in the released validation suite, are
consistent with this boundary and suggest a sharper conjecture:
symmetric unsharp instruments (noise level $\eta=0.3$; states and 
measurement axes swept) preserved QQ in all tested
configurations, while violations occurred in some tested asymmetric
variants. Further open problems include the representation question
(which observed kernel pairs admit a projective realization),
Bayesian-update boundaries (where standard belief-updating models
fall relative to the QQ surface), and higher-rank cyclic audits.

\subsection{Discriminant table}
\label{ss:discriminant}

Table~\ref{tab:discriminant} assembles the results of this section
into a two-layer interpretation aid for audit outcomes. The two
layers answer different questions and are evaluated independently.
The compatibility layer asks which mechanism families could have
produced an observed pattern, while the certification layer asks
whether the observed order dependence exceeds what noncontextual
couplings allow. An item therefore receives a certification verdict
whether or not any compatibility fit is attempted.

The compatibility layer maps diagnostic patterns to mechanism
classes that remain compatible with them. These statements are
non-identifying: Proposition~\ref{prop:symrep} establishes explicitly
that QQ satisfaction is not unique to projective mechanisms, while
the closed-form signatures likewise provide compatibility conditions
rather than converse identification results. The rows are neither
mutually exclusive nor exhaustive. 
The first row collects mechanism classes that satisfy QQ robustly, as
established by Theorems~\ref{thm:wb} and~\ref{thm:mi} and
Propositions~\ref{prop:symrep} and~\ref{prop:mixture}, together with
the order-matched mixing closure of Proposition~\ref{prop:mixing}.
For the mixing closure, robust QQ satisfaction is preserved when the
component mechanisms themselves satisfy QQ robustly and the mixture
weights are identical across the two orders.
The two closed-form rows summarize the polarity- and
order-dependent repetition signatures of
Proposition~\ref{prop:mixture}. Exercising either closed-form row
empirically requires an item family assumed to share the relevant
retention parameters, together with a pre-specified fitting and
rejection rule.

The certification layer applies the certification logic of
Corollary~\ref{cor:cbd} through the residual
\begin{equation*}
\Gamma \;:=\; |\qQQ| - \OSS .
\end{equation*}
In the adopted binary cyclic rank-2 CbD setting,
$\Gamma \le 0$ is equivalent to noncontextuality. Because the audit
observes $\Gamma$ only up to bounded measurement slack,
Section~\ref{sec:method} constructs certified bounds
$\Gamma_{\mathrm{lower}} \le \Gamma \le
\Gamma_{\mathrm{upper}}$. The certification layer then issues the
three-way verdict shown in the table: certified noncontextual when
$\Gamma_{\mathrm{upper}} \le 0$, certified residual contextuality
when $\Gamma_{\mathrm{lower}} > 0$, and indeterminate otherwise.

Section~\ref{sec:pilot} reports item-level QQ verdicts and applies
the certification layer of this table. The closed-form fits in the
compatibility layer were not part of the pilot's pre-specified
decision procedure, and no item family sharing the relevant retention
parameters was defined for them. The observed saturation would also
likely have limited the informativeness of such fits. The mechanisms
generating the observed violations therefore remain unresolved in
the pilot, while the closed-form rows remain theoretical diagnostic
branches for audits with sufficiently dispersed responses,
predefined shared-parameter families, and pre-specified fitting
criteria.

\begin{table}[t]
\centering\small
\begin{tabular}{p{0.43\linewidth}p{0.51\linewidth}}
\toprule
Diagnostic pattern & Compatibility or certification statement \\
\midrule
\multicolumn{2}{l}{\emph{Layer 1: mechanism compatibility}} \\
\midrule
QQ satisfaction across an item family &
marginal-independent mechanisms
($a{=}b{=}c{=}d$); projective mechanisms;
symmetric repetition and its cross-polarity, cross-order symmetric variant;
order-matched mixtures of robustly QQ-satisfying mechanisms\\
violations fitting
$q\sim(r_n{-}r_y)(p_A(y){-}p_B(y))$ &
polarity-dependent repetition \\
violations fitting
$q\sim(r_{BA}{-}r_{AB})S$ &
order-dependent repetition \\
\midrule
\multicolumn{2}{l}{\emph{Layer 2: CbD certification
(independent of Layer 1)}} \\
\midrule
certified $\Gamma_{\mathrm{upper}}\le 0$ &
certified noncontextual under the adopted rank-2 CbD criterion;
the generating mechanism is not identified \\
certified $\Gamma_{\mathrm{lower}}>0$ &
certified residual contextuality under the adopted rank-2
CbD criterion \\
$\Gamma_{\mathrm{lower}}\le 0
 < \Gamma_{\mathrm{upper}}$ &
indeterminate at the certification bounds \\
\bottomrule
\end{tabular}
\caption{Two-layer discriminant table. Layer~1 maps diagnostic
patterns to compatible mechanism classes
(Theorems~\ref{thm:wb} and~\ref{thm:mi}; Propositions~\ref{prop:symrep}, \ref{prop:mixture},
and~\ref{prop:mixing});
Layer~2 issues the CbD certification verdict of
Corollary~\ref{cor:cbd}, independently of Layer~1. In the closed
forms of Proposition~\ref{prop:mixture}, $r_y$ and $r_n$ are
polarity-dependent retention rates, $r_{AB}$ and $r_{BA}$ are
order-dependent retention rates, and
$S=p_A(y)p_B(n)+p_A(n)p_B(y)$.}
\label{tab:discriminant}
\end{table}

\section{Audit Methodology}\label{sec:method}

This section translates the theoretical distinctions above into a
reproducible audit pipeline for sequential judgments of LLMs. Its
core is a multi-turn forced-branch protocol that constructs the
sequential joint distributions of Section~\ref{ss:setup} from
next-token readouts under fixed implementation invariants. Around
this core, the pipeline controls label and unassigned-mass effects. 
It then checks both computational consistency and distributional
adequacy before QQ or contextuality verdicts are interpreted.
Figure~\ref{fig:pipeline} summarizes the pipeline and the outcomes each stage produced in this pilot.

\begin{figure}[!htbp]
\centering
\begin{tikzpicture}[
  font=\small,
  stage/.style={draw, rounded corners=2pt, align=left, text width=6.0cm,
                inner sep=6pt, fill=white},
  gatebox/.style={draw, thick, rounded corners=2pt, align=left,
                  text width=6.0cm, inner sep=6pt, fill=gray!10},
  badge/.style={draw, dotted, align=left, text width=4.4cm,
                inner sep=5pt, fill=gray!15, font=\footnotesize},
  auxbox/.style={draw, dashed, align=left, text width=2.9cm,
                 inner sep=5pt, font=\footnotesize},
  flow/.style={-{Stealth[length=2.4mm]}, thick},
  drop/.style={dotted, -{Stealth[length=2mm]}}
]
\node[stage] (reg) {\textbf{Pre-specified design:}
frozen manifest (hash-tracked); 18 item pairs (13 A/B, 5 Yes/No);
pre-specified rules and thresholds};
\node[stage, below=0.55cm of reg] (meas) {\textbf{Measurement:}
one framing per track: direct evaluation (Track~M) or
survey-respondent persona (Track~P); forced binary branching;
independent full-history forward passes; 12 positions per A/B item,
6 per Yes/No item};
\node[stage, below=0.55cm of meas] (rec) {\textbf{Joint reconstruction:}
branch-weighted token joints $p_A(\alpha)\,K_{AB}(\beta\,|\,\alpha)$,
both orders; semantic canonicalization; statistics $\qQQ$, $OE_A$,
$OE_B$, $\OSS$};
\node[gatebox, below=0.85cm of rec] (gates) {\textbf{Health gates:}
G1 labels single tokens (per track and label scheme); G2 unassigned mass (all
positions); G3 readout vs.\ sampling (pre-specified spot checks)};
\node[stage, below=0.55cm of gates] (diag) {\textbf{Diagnostics:}
mapping-specific and label diagnostics (descriptive);
counterbalanced pooling; saturation diagnostic
($\max(p,1-p)>0.99$)};
\node[stage, below=0.55cm of diag] (verd) {\textbf{Uncertainty and
verdicts:} unassigned mass to worst-case envelopes; QQ trichotomy
($\varepsilon_{\mathrm{QQ}}$ band);
$\Gamma=|\qQQ|-\OSS$ certification (discriminant table)};
\node[draw, dashed, rounded corners=3pt, inner sep=6pt,
      fit=(meas)(rec)] (band) {};
\node[font=\normalsize\itshape, rotate=90, anchor=south, align=center, text width=5.6cm]
      at ([xshift=-2pt]band.west)
      {invariants: fixed template, 1-token labels, no shared cache, batch 1};
\node[auxbox, left=0.7cm of gates] (halt)
      {on gate failure: halt, or restricted interpretation only};
\draw[flow, dashed] (gates.west) -- (halt.east);
\node[auxbox, left=0.7cm of diag] (memo)
      {auxiliary: memorization alert (name-substitution pairs)};
\draw[drop] (memo.east) -- (diag.west);
\node[badge, right=0.8cm of reg] (b1)
      {design frozen; item pairs and rules fixed before the runs they govern};
\node[badge, right=0.8cm of meas] (b2)
      {Track~M: 186 positions; Track~P: 96 positions};
\node[badge, right=0.8cm of rec] (b3)
      {full-precision joints stored in preserved audit records};
\node[badge, right=0.8cm of gates] (b4)
      {G1 passed (all labels); G2 max other mass $\sim 10^{-7}$;
       G3 agreement within Clopper--Pearson intervals (2 items)};
\node[badge, right=0.8cm of diag] (b5)
      {saturated 17/18 (M), 7/8 (P); label: 8/13 above 0.3
       (descriptive), 4 verdict changes; memorization alert not
       triggered (1/3)};
\node[badge, right=0.8cm of verd] (b6)
      {pooled: 13 V, 5 S, 0 I; no certified
       $\Gamma_{\mathrm{lower}}>0$; indeterminate: 2 (M), 3 (P)};
\draw[flow] (reg) -- (meas);
\draw[flow] (meas) -- (rec);
\draw[flow] (rec) -- (gates);
\draw[flow] (gates) -- (diag);
\draw[flow] (diag) -- (verd);
\foreach \s/\b in {reg/b1, meas/b2, rec/b3, gates/b4, diag/b5, verd/b6}
  \draw[drop] (\s.east) -- (\b.west);
\node[font=\Large\bfseries\itshape, anchor=south] at ([yshift=5pt]reg.north)
      {audit pipeline};
\node[font=\Large\bfseries\itshape, anchor=south] at ([yshift=5pt]b1.north)
      {pilot outcomes};
\end{tikzpicture}
\caption{The audit pipeline (left column) and the outcomes it
produced in this pilot (right column, dotted). Solid arrows give the
processing order. The dashed band marks the implementation invariants
enforced during measurement and reconstruction, the dashed branch
beside the health gates marks the pre-specified failure path, and the
memorization alert is an auxiliary diagnostic outside the main flow.
Gate denominators differ by design: G1 and G2 cover all items and
positions, while G3 is a pre-specified spot check on two items.
Abbreviations: V, VIOLATED; S, SATISFIED; I, INDETERMINATE.}
\label{fig:pipeline}
\end{figure}
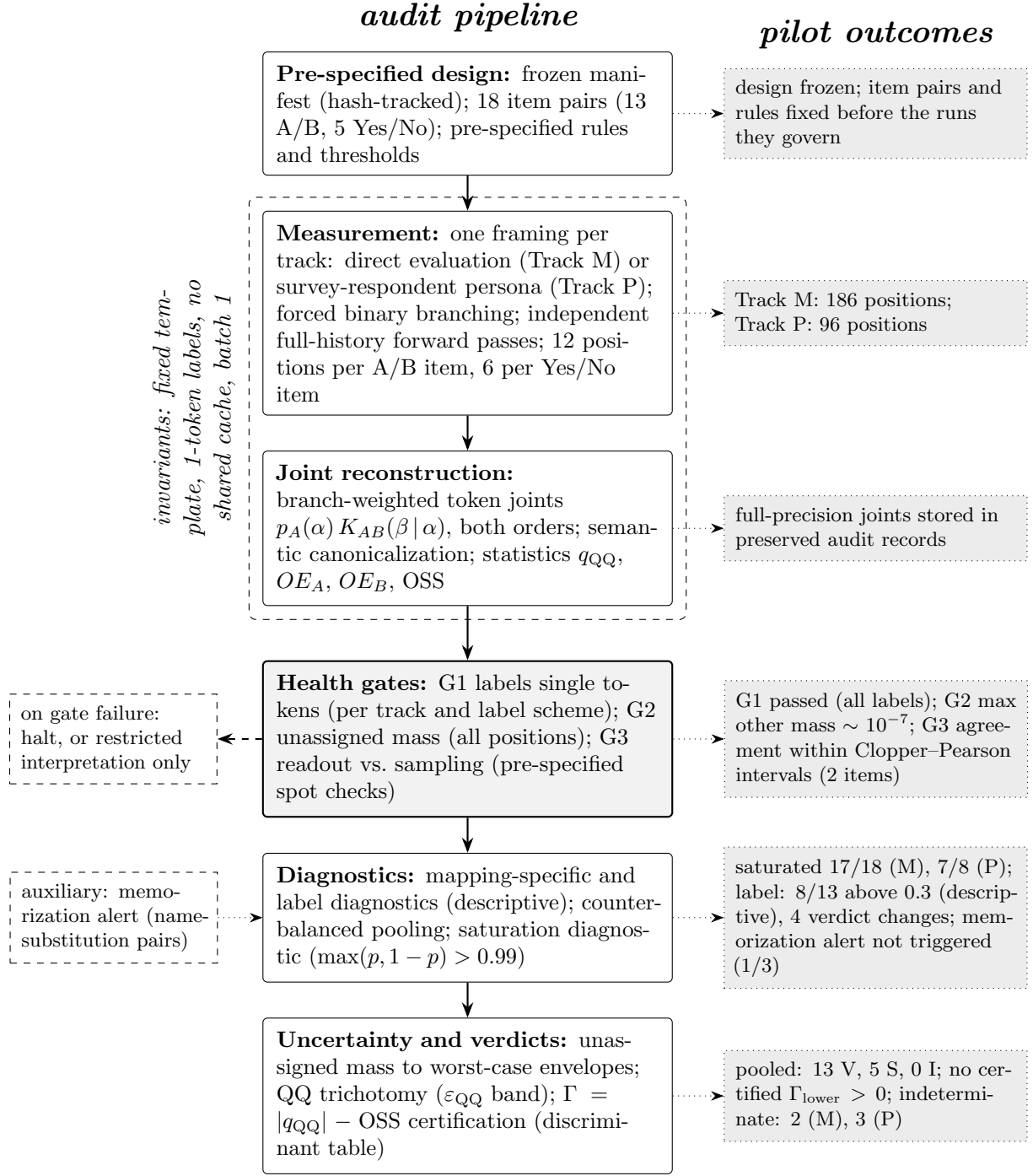

\paragraph{Estimand.} The \emph{binary-conditioned audit estimand} is
$P(y\mid y\text{ or }n)=e^{\ell_y}/(e^{\ell_y}+e^{\ell_n})$, where
$\ell_y$ and $\ell_n$ are the next-token log-probabilities of the two
canonical single-token labels at the answer position. This estimand
differs from full-vocabulary generation; the gap is tracked through
the unassigned probability mass (the complement of the two labels'
total next-token probability), recorded at every position.

\paragraph{Multi-turn forced-branch protocol.}
Sequential structure is what separates this audit from single-shot
first-token measurement, and it is realized through a genuine
multi-turn conversation. For each ordered pair, the protocol proceeds
in four steps: (1) the first question is presented as a user turn and
its binary-conditioned answer distribution is read out; (2) \emph{forced branching}: each
candidate answer is committed separately to the conversation history
as an assistant turn containing exactly the corresponding canonical
label token; (3) the second question is presented as a further user
turn, and its answer distribution is measured anew under each
committed branch; and (4) steps (1)--(3) are repeated with the question
order reversed. Forced branching is essential because it evaluates
both possible first-answer histories, including branches with very
small natural probability, while weighting each branch by its measured
first-answer probability when reconstructing the joint distribution.

The joint is therefore not formed by multiplying independently
measured marginals. Such a construction imposes conditional
independence between the two answers and cannot represent the
history-mediated dependence targeted by the audit. Nor is the exchange
concatenated into a single user prompt, because doing so changes the
chat-template serialization, the conversational roles, and hence the
measured conditional distribution. 
Let $H_{A_\alpha}$, $\alpha\in\{y,n\}$, denote the serialized
conversation history in which the first question $A$ is followed by
an assistant turn committing the canonical label token for answer
$\alpha$, and define $H_{B_\beta}$ analogously for the reverse
order. The branch-specific readouts are precisely the
transition-kernel entries introduced in Section~\ref{ss:setup},
\begin{equation}\label{eq:forced-branch-kernels}
P(B_\beta \mid H_{A_\alpha}) = K_{AB}(\beta\,|\,\alpha),
\qquad
P(A_\alpha \mid H_{B_\beta}) = K_{BA}(\alpha\,|\,\beta).
\end{equation}
The sequential joints are then assembled as in
Section~\ref{ss:setup}:
\begin{equation}\label{eq:sequential-joint-assembly}
p_{AB}(\alpha,\beta)=p_A(\alpha)\,K_{AB}(\beta\,|\,\alpha),
\qquad
p_{BA}(\beta,\alpha)=p_B(\beta)\,K_{BA}(\alpha\,|\,\beta).
\end{equation}
Committing the first answer to the history is the operational analogue
of the state update presupposed by the sequential mechanisms of
Section~\ref{sec:theory}. Without branch-conditioned multi-turn
histories, the target sequential joint distributions cannot be
measured under this audit, and the resulting order effects and QQ
statistics would not correspond to the intended sequential process.

\paragraph{Implementation invariants.} Four conditions are held fixed
across every measured position, and the audit's validity rests on
them. (1) A single chat template renders every conversation: all
branches, orders, mappings, and framings (the direct-evaluation and
survey-respondent persona instructions; Appendix~\ref{app:system-prompts})
are serialized by the same
template, whose hash is stored in the execution records, so that
measured differences cannot be attributed to changes in the
chat-template implementation. (2) Each canonical label is exactly one
token in the rendered context (enforced as Gate G1 below), and no
hidden generation prefix is active (thinking mode disabled), so that
the committed label token is the entire assistant turn and both
canonical answer probabilities are obtained from a single next-token
distribution. (3) Token-level results under each label
mapping are canonicalized back to the semantic option coordinates
before any pooling or comparison, so that all reported quantities
refer to semantic options rather than to the neutral tokens.
(4) Every position-level logprob readout is computed by an independent
forward evaluation of its full serialized history. Each forward
evaluates a single sequence (batch size one), and no state, cache, or
continuation is shared across measured positions, branches, or orders.
The rendered conversation format, including the per-turn label
legend, is reproduced in Appendix~\ref{app:presentation}.

\paragraph{Label counterbalancing.} Original survey wordings and
response options are preserved; options map to neutral tokens A/B
under \emph{both} assignments (map-1 and its swap, map-2), each applied identically
across orders and protocols, with the active mapping displayed as a
two-line legend in every user turn (Appendix~\ref{app:presentation}). 
Results are canonicalized back to semantic options
before pooling with equal weights. Pooling is justified by the linearity
of $\qQQ$ in the joint pair together with the pre-specified equal-weight
estimand; Proposition~\ref{prop:mixing} additionally clarifies how such
order-matched mixtures preserve QQ structure. Per-mapping results are
reported as label-effect diagnostics (the two mappings are experimental
conditions, not latent population components).

\paragraph{Decision rules (pre-specified; frozen before execution and
audit-logged).} 
Unassigned probability mass (``other'') is handled by a
worst-case robustness envelope. For a single label mapping,
$\qQQ$, $OE_A$, and $OE_B$ are multilinear in six
mass-allocation variables --- in each order, one at the
first-question position and one at each of the two
branch-conditioned second-question positions. Their exact component
ranges over the allocation box are therefore attained on its $2^6$
vertices. For the counterbalanced S1/S2 items, this enumeration is
performed separately after semantic canonicalization for each of the
two mappings. Because the mappings have independent allocation variables and the
pooled joint distributions use equal weights, the exact pooled range
of each linear audit component $X\in\{\qQQ,OE_A,OE_B\}$ has
endpoints
\begin{equation*}
X^{\mathrm{pool}}_{\min}
= \tfrac{1}{2}\bigl(X^{(1)}_{\min}+X^{(2)}_{\min}\bigr),
\qquad
X^{\mathrm{pool}}_{\max}
= \tfrac{1}{2}\bigl(X^{(1)}_{\max}+X^{(2)}_{\max}\bigr),
\end{equation*}
where $X^{(k)}_{\min}$ and $X^{(k)}_{\max}$ denote the envelope
endpoints of $X$ computed under mapping $k$.
For S3, which uses a single Yes/No mapping, the single-mapping
$2^6$ envelope applies directly.
The equivalence decision is taken over the envelope: SATISFIED if
the entire envelope lies within the practical-equivalence band
$|\qQQ|\le\varepsilon_{\mathrm{QQ}}=0.02$, VIOLATED if it lies entirely
outside, and INDETERMINATE if it straddles the boundary (all with a
floating-point tolerance).
The margin $\varepsilon_{\mathrm{QQ}}$ was fixed a priori as a
practical-equivalence band wide enough to cover approximate
satisfaction by partial-revision processes, for which simulations in
the Supporting Information of \citet{wang:2014:pnas} report $|q|<0.015$.

\paragraph{Certified $\Gamma$ bounds.}
The final certification estimand is constructed from the semantically
canonicalized pooled audit components, rather than by averaging
mapping-specific values of $\Gamma$ or mapping-specific certification
verdicts. For counterbalanced S1/S2 items,
\begin{equation*}
\qQQ^{\mathrm{pool}}
=\tfrac{1}{2}\bigl(\qQQ^{(1)}+\qQQ^{(2)}\bigr),
\qquad
OE_A^{\mathrm{pool}}
=\tfrac{1}{2}\bigl(OE_A^{(1)}+OE_A^{(2)}\bigr),
\qquad
OE_B^{\mathrm{pool}}
=\tfrac{1}{2}\bigl(OE_B^{(1)}+OE_B^{(2)}\bigr).
\end{equation*}
Because the absolute-value operations do not commute with pooling,
\begin{equation*}
\Gamma^{\mathrm{pool}}
=
|\qQQ^{\mathrm{pool}}|
-|OE_A^{\mathrm{pool}}|
-|OE_B^{\mathrm{pool}}|
\neq
\tfrac{1}{2}\bigl(\Gamma^{(1)}+\Gamma^{(2)}\bigr)
\end{equation*}
in general.

Let $I_q$, $I_A$, and $I_B$ denote the exact pooled component
envelopes for $\qQQ$, $OE_A$, and $OE_B$, respectively; for S3,
they denote the corresponding exact single-mapping envelopes. For an
interval $I=[\ell,u]$, define
\begin{equation*}
L_{\mathrm{abs}}(I)
=
\begin{cases}
0, & \ell\le 0\le u,\\
\min\{|\ell|,|u|\}, & \text{otherwise},
\end{cases}
\qquad
U_{\mathrm{abs}}(I)
=
\max\{|\ell|,|u|\}.
\end{equation*}
Because
$\Gamma=|\qQQ|-\OSS=|\qQQ|-|OE_A|-|OE_B|$
is only piecewise multilinear, certified bounds are assembled from
these component envelopes as
\begin{align*}
\Gamma_{\mathrm{lower}}
&=
L_{\mathrm{abs}}(I_q)
-U_{\mathrm{abs}}(I_A)
-U_{\mathrm{abs}}(I_B),\\
\Gamma_{\mathrm{upper}}
&=
U_{\mathrm{abs}}(I_q)
-L_{\mathrm{abs}}(I_A)
-L_{\mathrm{abs}}(I_B).
\end{align*}
These bounds are sound but may be conservative because the dependence
among the three component envelopes is not retained. They are the
$\Gamma$-envelope bounds consumed by the certification layer of the
discriminant table (Section~\ref{ss:discriminant}): an item is
certified noncontextual when $\Gamma_{\mathrm{upper}}\le 0$,
certified residually contextual when
$\Gamma_{\mathrm{lower}}>0$, and indeterminate otherwise.

\paragraph{Saturation diagnostic (pre-specified).} An item is
\emph{saturated} if more than half of its measured positions have
$\max(p,1-p)>0.99$. Saturated items carry little distributional
information; their QQ verdicts are reported but weakly interpretable.

\paragraph{Health gates.} Three gates screen technical failure modes
before any verdict is interpreted. Gate G1 requires the canonical
labels to be single tokens in the rendered chat-template context, so
that both canonical answer probabilities can be read from a single 
next-token distribution. 
G1 was validated once per track and label scheme in a
representative rendered context. Because the answer-position
template boundary is identical across items within each track and label scheme, the
validation applies to every item.
Gate G2 requires the unassigned mass to stay below $10^{-2}$ at 
every position, screening gross mass leakage; the decision-relevant
quantity is the resulting envelope width, which is reported with the
results. 
Gate G3 cross-checks the logprob readout against trajectory-level
rejection sampling by exact Clopper--Pearson
intervals~\citep{clopper:1934} with Bonferroni correction (per
item-order, $\alpha=0.01$), guarding against readout--generation
divergence; the method was selected after an observed simultaneous
undercoverage of Wilson intervals ($0.980$ vs $0.995$ at $n=200$ in
calibration). 
It is run as pre-specified spot checks rather than exhaustively, on
items fixed in advance from the frozen manifest to cover both
canonical-label schemes (Appendix~\ref{app:item-manifest}). Under the
direct-evaluation framing, the check covered items s1-01 and s3-01
--- one per label scheme (A/B and Yes/No, respectively); under the
survey-respondent persona framing, which applies only to A/B items,
item s1-01. In every case both presentation orders were checked, with
200 accepted joint samples collected per item and order.

The original spot checks used batched sampling (batch size 64). A
subsequent implementation audit on a later software stack found
batch-size-dependent generation distortion (in a fixed-input
reproduction test, single-sequence generation matched the forward
pass while larger batches did not). As a conservative recheck, all
six pre-specified item-order checks (four in Track~M, two in Track~P) were
re-executed at batch size one using the pilot model and the original protocol of each track,
and every verdict was preserved; the retrospective records are
preserved alongside the original audit records.

\paragraph{Software validation.} The mathematical core of the audit software
was validated against ground-truth mechanisms
(projective, repetition families, L\"uders instruments), recovering
$q=0$ at machine precision and all closed-form violations exactly; the
full software assertion suite ships with the code as part of the
reproducibility materials.

\section{Empirical Pilot}\label{sec:pilot}

This section applies the audit pipeline to a first-signal pilot on an
instruction-tuned LLM, evaluating both the technical reliability of the
procedure and whether the resulting response distributions are adequate
for meaningful QQ and CbD analysis.

\paragraph{Setup.} All measurements were taken on
Qwen3-4B-Instruct-2507~\citep{yang:2025:qwen3}, run in bfloat16 (bf16) in
non-thinking mode; this is a single-model study, and no cross-model
generalization is claimed (see Section~\ref{sec:limitations}). The item set consists of
18 question pairs: 10 historical survey pairs with original wordings
and response options, drawn from \citet{wang:2014:pnas} and related
classic survey sources (Gallup, Pew, SRC), forming stratum S1;
3 name-swapped memorization controls, forming stratum S2; and 5 novel judgment pairs
answered with Yes/No labels, forming stratum S3. Two framings were tested, labeled
Track~M and Track~P, keeping the names used in the pre-specified protocol and its audit
records. Track~M applies a direct evaluation instruction; 
Track~P applies a survey-respondent persona with a per-item reference
year. The persona requires an original survey year, so Track~P covers
only the 8 year-fixed items in strata S1 and S2; it was run once,
with the framing redesigned and all other elements of the protocol
held fixed, and its full specification was frozen before execution.
We use M and P as shorthand in tables and
figures; the full prompts and the item manifest are given in
Appendix~\ref{app:materials}. 
Measurement was deterministic at the logprob level; the sampling
checks drew tokens with temperature~1 and top-$p$~1, i.e., from the
unmodified model distribution that the logprob readout records.
Every run writes a preserved, hash-tracked audit
record containing the model and tokenizer revision fields, the chat-template
hash, label token ids, mapping assignment, seeds, and rejection
attempt and accept counts.

\paragraph{Gate results.} Gates G1 and G2 passed in every run, and
gate G3 passed in all pre-specified spot checks: the maximum other
mass was $\sim2\times10^{-7}$, and the consistency checks passed
without retries. 
The gate results reduce concern that the patterns reported below
arise from basic, pre-specified failures of the measurement chain.
What follows therefore concerns the adequacy of the
model--prompt--item--estimand combination for measuring
population-like dispersed response distributions, rather than the
correctness of the implementation as screened by the gates.

\paragraph{Results.} Table~\ref{tab:trackm} reports the pooled
Track~M verdicts for all 18 items; Table~\ref{tab:mp} compares the
eight items measured under both framings; Table~\ref{tab:label}
tabulates the per-mapping label-effect diagnostics, and
Table~\ref{tab:s107-joint} shows the order-conditioned joint
distributions of one illustrative item; Figure~\ref{fig:saturation}
visualizes the position-level saturation, and Figure~\ref{fig:label}
the label effect. Five findings follow, then the main methodological
conclusion they support.

\begin{table}[!htbp]
\centering\small
\setlength{\tabcolsep}{10pt}
\begin{tabular}{lrrrrllc}
\toprule
id & $OE_A$ & $OE_B$ & $\OSS$ & $\qQQ$ & QQ & $\Gamma$ class & sat \\
\midrule
s1-01 & $-0.280$ & $+0.462$ & 0.742 & $+0.182$ & VIOLATED & NONCONTEXTUAL & Y \\
s1-02 & $+0.482$ & $+0.498$ & 0.981 & $+0.981$ & VIOLATED & \textbf{INDETERMINATE} & Y \\
s1-04 & $+0.216$ & $-0.190$ & 0.406 & $+0.001$ & SATISFIED & NONCONTEXTUAL & Y \\
s1-05 & $-0.001$ & $-0.731$ & 0.733 & $-0.730$ & VIOLATED & NONCONTEXTUAL & Y \\
s1-06 & $-0.007$ & $+0.002$ & 0.009 & $-0.006$ & SATISFIED & NONCONTEXTUAL & Y \\
s1-07 & $-0.387$ & $+0.387$ & 0.774 & $-0.000$ & SATISFIED & NONCONTEXTUAL & Y \\
s1-09 & $-0.131$ & $-0.178$ & 0.309 & $-0.309$ & VIOLATED & NONCONTEXTUAL & Y \\
s1-10 & $+0.001$ & $+0.024$ & 0.025 & $-0.023$ & VIOLATED & NONCONTEXTUAL & Y \\
s1-11 & $-0.346$ & $-0.395$ & 0.741 & $-0.049$ & VIOLATED & NONCONTEXTUAL & Y \\
s1-12 & $+0.005$ & $+0.000$ & 0.006 & $+0.005$ & SATISFIED & NONCONTEXTUAL & Y \\
s2-01 & $+0.493$ & $-0.311$ & 0.804 & $+0.182$ & VIOLATED & NONCONTEXTUAL & Y \\
s2-02 & $-0.000$ & $+0.000$ & 0.000 & $+0.000$ & SATISFIED & NONCONTEXTUAL & Y \\
s2-09 & $-0.051$ & $-0.544$ & 0.595 & $+0.468$ & VIOLATED & NONCONTEXTUAL & Y \\
s3-01 & $+0.000$ & $-0.562$ & 0.562 & $+0.562$ & VIOLATED & \textbf{INDETERMINATE} & Y \\
s3-03 & $-0.007$ & $+0.493$ & 0.500 & $+0.500$ & VIOLATED & NONCONTEXTUAL & Y \\
s3-08 & $+0.378$ & $+0.991$ & 1.369 & $-0.614$ & VIOLATED & NONCONTEXTUAL & Y \\
s3-09 & $-0.044$ & $-0.731$ & 0.775 & $+0.687$ & VIOLATED & NONCONTEXTUAL & Y \\
s3-11 & $-0.924$ & $-0.034$ & 0.959 & $-0.890$ & VIOLATED & NONCONTEXTUAL & N \\
\bottomrule
\end{tabular}
\caption{Track~M pooled results (18 items). $OE_A$ and $OE_B$ are the marginal order effects entering $\OSS=|OE_A|+|OE_B|$, with signs as defined in Section~\ref{ss:setup}. The sat column marks item-level saturation (Y/N) under the pre-specified rule of Section~\ref{sec:method}. S1/S2 use A/B-mapped original options; S3 uses Yes/No --- metrics are not compared numerically across strata. Envelope intervals are indistinguishable from point values at the displayed precision (other mass $\sim10^{-7}$). Classifications use full-precision certified bounds; displayed values are rounded.} 
\label{tab:trackm}
\end{table}

\begin{table}[!htbp]
\centering\small
\setlength{\tabcolsep}{11pt}
\begin{tabular}{llllllc}
\toprule
id & year & sat M$\to$P & $\qQQ$ M$\to$P & $\OSS$ M$\to$P & QQ (P) & $\Gamma$ (P) \\
\midrule
s1-01 & 1997 & Y$\to$Y & $+0.18\to-0.57$ & $0.74\to0.57$ & VIOLATED & NC \\
s1-02 & 1995 & Y$\to$Y & $+0.98\to+0.11$ & $0.98\to0.11$ & VIOLATED & \textbf{IND} \\
s1-04 & 1996 & Y$\to$Y & $+0.00\to+0.65$ & $0.41\to0.65$ & VIOLATED & NC \\
s1-06 & 1979 & Y$\to$Y & $-0.01\to-1.00$ & $0.01\to1.00$ & VIOLATED & NC \\
s1-11 & 2008 & Y$\to$N & $-0.05\to-0.24$ & $0.74\to0.89$ & VIOLATED & NC \\
s1-12 & 2012 & Y$\to$Y & $+0.01\to+0.00$ & $0.01\to0.00$ & SATISFIED & \textbf{IND} \\
s2-01 & 1997 & Y$\to$Y & $+0.18\to-1.00$ & $0.80\to1.00$ & VIOLATED & \textbf{IND} \\
s2-02 & 1995 & Y$\to$Y & $+0.00\to+0.01$ & $0.00\to0.47$ & SATISFIED & NC \\
\bottomrule
\end{tabular}
\caption{Track~M vs Track~P, 8 year-fixed pairs. Under Track~P no
item had certified $\Gamma_{\mathrm{lower}}>0$ (5 certified
noncontextual (NC), 3 indeterminate (IND) at the certification bounds). Classifications use full-precision certified bounds; displayed values are rounded.}
\label{tab:mp}
\end{table}

\begin{table}[!htbp]
\centering\small
\setlength{\tabcolsep}{11pt}
\begin{tabular}{lrrrrrrcc}
\toprule
id & $\qQQ^{(1)}$ & $\qQQ^{(2)}$ & $|\Delta\qQQ|$ & $\OSS^{(1)}$ &
$\OSS^{(2)}$ & $|\Delta\OSS|$ & QQ$^{(1)}$ & QQ$^{(2)}$ \\
\midrule
s1-01 & $-0.560$ & $+0.924$ & 1.484 & 0.560 & 0.924 & 0.364 & V & V \\
s1-02 & $+0.963$ & $+0.998$ & 0.036 & 0.963 & 0.998 & 0.036 & V & V \\
s1-04 & $+0.000$ & $+0.003$ & 0.003 & 0.006 & 0.818 & 0.812 & S & S \\
s1-05 & $-0.962$ & $-0.498$ & 0.465 & 0.963 & 0.502 & 0.460 & V & V \\
s1-06 & $-0.014$ & $+0.002$ & 0.016 & 0.014 & 0.004 & 0.010 & S & S \\
s1-07 & $-0.000$ & $+0.000$ & 0.000 & 1.358 & 0.189 & 1.169 & S & S \\
s1-09 & $-0.618$ & $+0.000$ & 0.618 & 0.618 & 0.004 & 0.614 & V & S \\
s1-10 & $-0.046$ & $-0.000$ & 0.045 & 0.049 & 0.000 & 0.049 & V & S \\
s1-11 & $-0.098$ & $+0.001$ & 0.099 & 1.456 & 0.026 & 1.431 & V & S \\
s1-12 & $+0.000$ & $+0.011$ & 0.011 & 0.000 & 0.011 & 0.011 & S & S \\
s2-01 & $+0.363$ & $-0.000$ & 0.363 & 1.562 & 0.046 & 1.516 & V & S \\
s2-02 & $+0.000$ & $+0.000$ & 0.000 & 0.001 & 0.000 & 0.001 & S & S \\
s2-09 & $-0.024$ & $+0.959$ & 0.983 & 0.215 & 0.976 & 0.762 & V & V \\
\bottomrule
\end{tabular}
\caption{Label-effect diagnostics for the 13 A/B items under
Track~M. Superscripts (1) and (2) denote map-1 and map-2, the two
counterbalanced label assignments. The last two columns give the
mapping-specific QQ verdicts (V, VIOLATED; S, SATISFIED); for four
items (s1-09, s1-10, s1-11, s2-01) the verdict changes across
mappings. S3 items use a single Yes/No arrangement and have no
mapping pair. All columns are rounded independently from
full-precision values, so displayed differences may deviate from
differences of displayed values in the last digit.}
\label{tab:label}
\end{table}

\begin{table}[!htbp]
\centering\small
\setlength{\tabcolsep}{12pt}
\begin{tabular}{lrrrr}
\toprule
order & $p(y,y)$ & $p(y,n)$ & $p(n,y)$ & $p(n,n)$ \\
\midrule
$AB$ & 0.387 & 0.000 & 0.000 & 0.613 \\
$BA$ & 0.000 & 0.000 & 0.000 & 1.000 \\
\bottomrule
\end{tabular}
\caption{Pooled joint answer distributions of item s1-07 under the two question
orders.  Cells are indexed by (first answer, second answer) within each order,
so the $AB$ row reads $(\alpha,\beta)$ and the $BA$ row $(\beta,\alpha)$, with
$\alpha$ the answer to $A$ and $\beta$ the answer to $B$.  Under order $BA$ the
distribution is fully deterministic at $(n,n)$, while under order $AB$ a $(y,y)$
cell of mass $0.387$ opens.  Reading off the identity in the text, the
joint-cell term $2\,[\,p_{AB}(y,y)-p_{BA}(y,y)\,]=0.77$ matches the marginal
term $OE_B-OE_A=0.77$, producing $q\approx0$. Cell values are rounded from full
precision.}
\label{tab:s107-joint}
\end{table}

\begin{figure}[!htbp]
\centering
\includegraphics[width=0.85\linewidth]{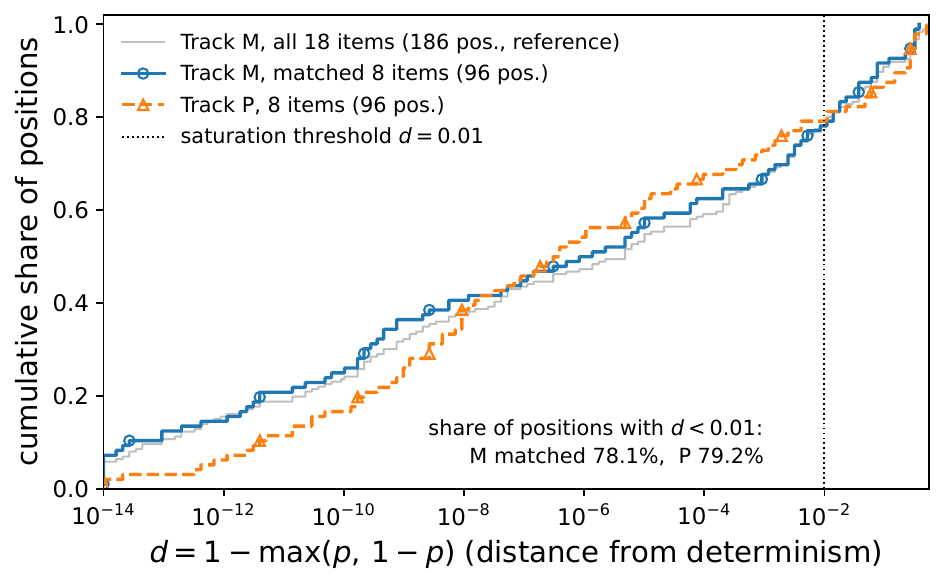}
\caption{Empirical cumulative distribution function (CDF) of the
distance from determinism at each measured position, shown on a log
scale. For a position with binary-conditioned probability $p$, the
distance is $d = 1-\max(p,\,1-p)$: the probability of the minority
answer, or equivalently the total-variation distance to the nearest
deterministic distribution, so that $d=0.5$ is maximal dispersion and
$d\to 0$ is saturation. The dotted line marks the saturation
threshold $d=0.01$, equivalent to the rule
$\max(p,1-p)>0.99$ of Section~\ref{sec:method}. The primary
comparison is between the matched eight-item subset under Track~M
(solid blue, circles) and Track~P (dashed orange, triangles), 96
positions each; the full 18-item Track~M set (gray) is a reference.
On the matched subset, a similar share of positions lies below the
threshold under both framings ($78.1\%$ vs $79.2\%$): the
respondent-persona framing shifted how deeply positions saturate ---
the curves differ at intermediate scales --- but did not
substantially reduce the share of saturated positions. Positions with
$d=0$ are clipped to $10^{-14}$ for display.}
\label{fig:saturation}
\end{figure}

\begin{figure}[!htbp]
\centering
\includegraphics[width=0.6\linewidth]{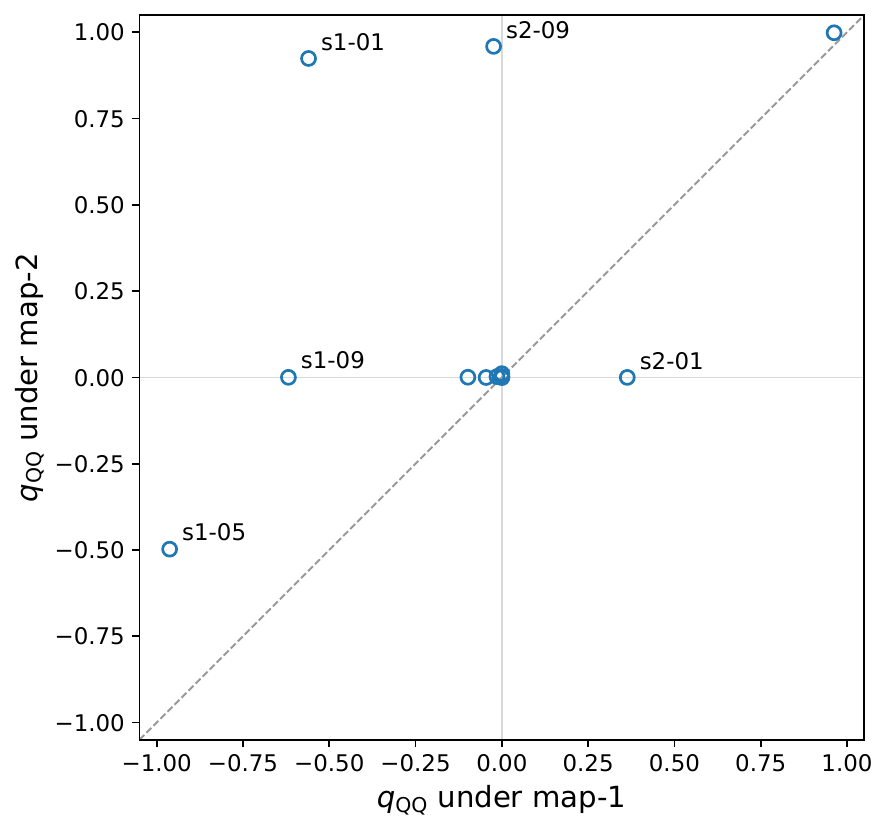}
\caption{Label effect under counterbalancing: per-item $\qQQ$ under map-1 vs map-2 (13 A/B items, Track~M). Points on the dashed diagonal indicate mapping-invariant behavior; annotated items deviate by more than $0.3$, showing that in the saturated regime the A/B token assignment has a substantive effect on the measured $\qQQ$.}
\label{fig:label}
\end{figure}

\begin{enumerate}
\item \emph{Saturation dominates, and persists under persona framing.}
Under Track~M (direct evaluation), 17 of the 18 items were saturated. 
The persona framing of Track~P
moved individual values dramatically --- s1-06's $q$, for example,
shifted from $-0.01$ to $-1.00$, the theoretical extremum --- but it
did not open the distributions (Figure~\ref{fig:saturation}): 7 of 8 
items remained saturated, with only s1-11 de-saturating (Table~\ref{tab:mp}).
We therefore report that saturation \emph{persisted} under respondent-persona
framing within the tested conditions, rather than claiming strict
invariance, since one item de-saturated and values moved substantially.
The verdict is insensitive to the saturation-threshold choice: at
thresholds $0.95$, $0.975$, and $0.99$ on $\max(p,1-p)$
(equivalently $d=0.05$, $0.025$, and $0.01$ in
Figure~\ref{fig:saturation}), the saturated-item counts are $18/18$,
$18/18$, and $17/18$ under Track~M, and $7/8$ throughout under Track~P.

\item \emph{The QQ verdicts conceal distinct saturated-regime
structures and mapping effects.} The pooled Track~M verdicts are 13
VIOLATED, 5 SATISFIED, and 0 INDETERMINATE, but these counts do not
represent a single underlying response mechanism. Three recurring
$(|q|,\OSS)$ patterns are visible in Table~\ref{tab:trackm}. Near-null items have
$|q|\approx\OSS\approx0$ (e.g., s1-06, s1-12, s2-02), indicating
little marginal order sensitivity and little QQ imbalance.
Near-boundary items have $|q|\approx\OSS$ (e.g., s1-05, s3-03,
s3-11), so the observed QQ imbalance is closely matched by the
magnitude of the marginal order effects and leaves little or no
residual excess beyond the noncontextual boundary. 
In contrast, large-order-effect items with small QQ imbalance have
$|q|\ll\OSS$ (s1-04, s1-07, s1-11), with s3-08 an intermediate case
($|q|/\OSS=0.45$). In these items substantial marginal shifts
coexist with a small QQ imbalance, and the two SATISFIED
cases (s1-04, s1-07) show that QQ satisfaction does not imply the
absence of order effects.
Cutting across these patterns, label-discordant items exhibit sharply different
mapping-specific values (Table~\ref{tab:label}, Figure~\ref{fig:label}), 
so their pooled statistics can be dominated
by mixture artifacts rather than by a stable semantic effect. For
example, s1-01 gives $q=-0.56$ under map-1 and $+0.92$ under map-2,
pooling to $+0.18$, with s2-01 and s2-09 showing similar discordance.
Thus, the 13 pooled QQ violations should not be interpreted as 13
instances of a common mechanism, and still less as 13 detections of
residual contextuality.

\item \emph{Item s1-07 separates order sensitivity, QQ imbalance,
and residual contextuality.} As shown in
Table~\ref{tab:trackm}, its marginal order effects are large and
opposite ($OE_A=-0.39$ and $OE_B=+0.39$), yet it satisfies QQ with
$q\approx0$ despite $\OSS=0.77$. This contrast arises because $q$ is
not determined by the marginal order effects alone. 
In each order the two answers disagree unless
they coincide, so the disagreement probability decomposes as
\[
M_{XY} = p_{XY}(A_y) + p_{XY}(B_y) - 2\,p_{XY}(y,y),
\]
the marginal masses counting the $(y,y)$ cell twice. Subtracting the
two orders and applying the sign conventions of
Section~\ref{ss:setup} gives the exact identity
\[
q = OE_B - OE_A - 2\,\bigl[p_{AB}(y,y) - p_{BA}(y,y)\bigr].
\]
For s1-07, the marginal contribution is
$OE_B-OE_A=0.77$, whereas the joint-cell contribution is
$2(p_{AB}(y,y)-p_{BA}(y,y))\approx0.77$. The two terms therefore offset one another,
yielding $q\approx0$. The cancellation occurs between the marginal
and joint-distribution contributions, not between $OE_A$ and $OE_B$.
This structure cannot be inferred from the marginal effects alone and
becomes visible only at the level of the full $2\times2$ joint
distributions (Table~\ref{tab:s107-joint}).

At the level of the QQ constraint, the observed satisfaction does not
exclude the projective class, but it does not establish that the full
joint pair admits a projective realization. In the saturated regime,
near-deterministic response patterns can produce the same QQ signature,
so the verdict alone neither identifies the generating mechanism nor
certifies contextuality.
Indeed, the discriminant analysis places s1-07 well inside the noncontextual
region, with $\Gamma_{\mathrm{upper}}=-0.77$. Thus, order sensitivity, QQ
imbalance, and residual contextuality are distinct properties: an item may
exhibit substantial order effects, satisfy QQ, and remain certifiably
noncontextual at the same time.

\item \emph{Label assignment has a substantive effect.} Across the
counterbalanced mappings, 8 of the 13 A/B items show $|\Delta\qQQ|$
or $|\Delta\OSS|$ above a descriptive threshold of $0.3$, with
per-item $\qQQ$ differing by up to 1.48 (s1-01) and $\OSS$ by up to
1.52 (s2-01), and for 4 items (s1-09, s1-10, s1-11, s2-01) the
mapping-specific QQ verdict changes between VIOLATED and SATISFIED
(Table~\ref{tab:label}). Figure~\ref{fig:label} visualizes the
mapping effect on $\qQQ$. In the saturated regime the A/B label
assignment can reverse near-deterministic binary choices at
individual measured positions. Counterbalancing with semantic
canonicalization is therefore not optional for audits of this kind,
consistent with the option-order and token-label biases reported in
single-shot settings~\citep{zheng:2024:llm-mcq-bias,pezeshkpour:2024:llm-mcq-sens}.

\item \emph{The memorization alert was not triggered.} The
pre-specified rule raises an alert if $|\OSS_{S1}-\OSS_{S2}|>0.3$ on
at least 2 of the 3 matched S1/S2 pairs; this occurred on only 1 of 3
(s1-02/s2-02), so the alert was not triggered. This check is
descriptive; comparison against human joint benchmarks remains
pending.
\end{enumerate}

\paragraph{Main finding (methodological).} Under the tested
instruction-tuned model and prompting conditions, forced-binary
next-token log-probabilities (i) \emph{revealed substantial order and
label sensitivity}, since the deterministic flips and label-mapping
effects above are observed properties of the binary-conditioned
estimand; (ii) \emph{were inadequate for measuring population-like
dispersed response distributions}, since the binary-conditioned
distribution was nearly deterministic under both framings; and (iii)
\emph{were consequently only weakly informative for fine-grained
mechanism discrimination and residual-contextuality detection}, since
saturation left distinct generating mechanisms observationally
compatible with the same verdicts, and mapping dependence made pooled
values unstable as evidence of any single semantic effect. 
The audit pipeline itself provided the intended diagnostics. The
health gates reduced concern that the result arose from the
pre-specified basic failure modes of the measurement chain.
Counterbalancing exposed label effects that a single-mapping audit
would have missed, the pre-specified saturation diagnostic converted
an otherwise overinterpretable result into an explicit measurement
caveat, and the combined QQ--CbD audit coordinates separated order
sensitivity, QQ imbalance, and residual contextuality on individual
items.
As a general (not universally proven) methodological warning, we recommend a
pre-specified saturation diagnostic as a standard health check for any study
treating LLM next-token distributions as survey-response distributions.

\FloatBarrier
\section{Discussion}

This section interprets the pilot findings, relating them to
LLM-as-respondent research, weighing candidate explanations for the
observed saturation, and drawing out what the results imply for QQ
and CbD audits. It then turns to design, outlining measurement
approaches that may yield more informative response distributions and
identifying which components of the audit engine transfer beyond this
pilot.

\paragraph{Relation to LLM-as-respondent studies.} 
Studies that aggregate LLM ``response distributions'' from forced-choice
next-token measurements implicitly assume dispersion. Our results
show that for a modern instruction-tuned model this assumption can
fail for nearly all tested items at the level of a single model under
fixed prompting conditions, with apparent distributional effects
driven by deterministic flips and label-assignment effects. This
sharpens, at the distribution level, the first-token critique of
\citet{wang:2024}.

\paragraph{Why saturation may arise.}
The pilot was not designed to identify the cause of the observed
saturation, and several candidate explanations remain observationally
compatible with the data. Instruction-following post-training may
encourage committed, low-entropy responses to judgment questions,
consistent with reports
that preference-based post-training reduces output
diversity~\citep{kirk:2024:rlhf-diversity,padmakumar:2024:content-diversity}
and can degrade the calibration of answer-token
probabilities~\citep{tian:2023:calibration}, even though pretrained
models are often well calibrated on comparable multiple-choice
formats~\citep{kadavath:2022:know}.
The single-token A/B labels may also carry prior asymmetries, which
counterbalancing reveals and averages over but does not eliminate
within individual mappings. Forced-binary conditioning can further
sharpen the reported distribution: renormalization preserves the odds
between the two labels but magnifies their absolute probability
separation whenever substantial probability mass lies outside the
two-token set. 
In this pilot the unassigned mass was negligible
($\sim 2\times10^{-7}$, Section~\ref{sec:pilot}), so this channel
contributed little here, though it remains relevant for formats with
weaker label anchoring.
Finally, particular item--prompt combinations may
elicit genuinely strong semantic preferences that would appear
near-deterministic across multiple response formats. These
explanations are not mutually exclusive. Distinguishing among them
requires factorial designs that vary post-training stage, response-label
format, and conditioning scheme. The present pilot held these factors
fixed, counterbalanced only the A/B label assignment, and varied item
content and two prompt framings.

\paragraph{What a QQ audit needs instead.}
The theory clarifies both an opportunity and a hazard. Human
aggregates are mixtures of individuals, and if individual-level
behavior satisfies QQ, Proposition~\ref{prop:mixing} explains why
aggregate satisfaction is preserved under order-matched mixing with
the same component weights across question orders. The reverse
implication does not hold, which matters for measurement design.
Different interventions may yield or expose numerical dispersion, but
they do not produce the same kind of variability. Persona or model
ensembles can introduce between-component heterogeneity, providing an
operational analogue of population heterogeneity in human surveys.
Model variants such as base checkpoints may expose different response
regimes but do not themselves define a population unless combined
under an explicit mixture design. By contrast, repeated sampling from
a single model produces within-component stochastic variability. Such
sampling may yield a dispersed empirical distribution, but it does not
by itself emulate a heterogeneous respondent population.

Mixture-based audits also carry a converse hazard. Under order-matched mixing,
the pooled QQ imbalance is the weighted average of the component-level imbalances,
so violations of opposite sign can cancel even when the aggregate satisfies QQ.
Aggregate verdicts should therefore be reported together with component-level
diagnostics and an explicit account of how components and mixture weights were
defined.

\paragraph{What the pilot does and does not show.}
This pilot documents a measurement caveat within a deliberately
narrow empirical scope. The scope of the empirical study and its
associated limitations are detailed in
Section~\ref{sec:limitations}. The persistence of saturation should
therefore be interpreted as a property of the tested
model--prompt--item--estimand configuration, not as evidence that
instruction-tuned LLMs generally produce saturated response
distributions. What the pilot does show is that an audit can pass the
pre-specified health gates and internal-consistency checks yet still
be limited by near-deterministic outputs.

In addition, the pilot does not show that LLMs are, or are not,
``quantum-like''. QQ satisfaction is compatible with classical
repetition (Proposition~\ref{prop:symrep}), QQ violation is
compatible with noncontextual couplings that allow direct influences
(Corollary~\ref{cor:cbd}), and in the saturated regime many verdicts
are dominated by deterministic or label-sensitive behavior.

The pilot makes this nonidentifiability concrete at the level of
individual items. Item s1-07, the cleanest of the three items with
$|q|\ll\OSS$, exhibits large and opposite marginal order effects,
satisfies QQ, and is certified noncontextual with a wide margin, so
order sensitivity, QQ satisfaction, and residual-contextuality
certification separate cleanly in a single item. The general lesson
is that these are three distinct properties, and none of the three
identifies a generating mechanism. This non-equivalence is already
implied by the characterization results of Section~\ref{sec:theory}
and is realized empirically here. 
In particular, QQ satisfaction alone does not identify the
projective mechanism. The same equality is reproduced exactly by
symmetric repetition (Proposition~\ref{prop:symrep}) and, in the
saturated regime, can also be reproduced by near-deterministic
response patterns.
An audit that reports QQ
verdicts without the accompanying $\OSS$ coordinate and certification
bounds would be unable to distinguish these properties, which is
precisely what the combined QQ--CbD audit coordinates of
Section~\ref{sec:method} are designed to prevent.

The informative content of the pilot therefore concerns the
measurement layer more than the mechanism layer. The absence of
certified contextuality here is qualitatively consistent with the zero
corrected CbD degree reported in the concurrent enterprise-agent audit
of \citet{tuan:2026}, but differs from the irreducible
context-dependence reported by
\citet{kumar:2026:fail-invariance-llm} in a pronoun-inference setting.
This comparison concerns the resulting CbD classification; our data do
not identify direct influence as the generating mechanism.
Task family, estimand construction, and the dispersion of the underlying
distributions all differ across these studies, and identifying the conditions
under which each regime emerges is a natural target for the
distribution-restoring designs above.

\paragraph{For the audit engine.}
The components that remain useful beyond this pilot are the
characterization theorems, the discriminant table, and the audit
pipeline with its pre-specified health gates. These contributions do
not depend on the particular LLM tested here, although their direct
application remains tied to the binary sequential setting and to the
definitions of $q$, $\OSS$, and $\Gamma$. 
Section~\ref{sec:pilot} records what each component
contributed in this pilot; here we consider what transfers.

What transfers is a reusable audit template rather than a
model-specific result. Saturation screening can be applied whenever
next-token probabilities are interpreted as response distributions.
Counterbalanced mappings, together with the semantic canonicalization
that keeps mapping-specific and pooled estimands directly comparable,
are relevant to forced-choice designs using arbitrary response
tokens. Worst-case envelopes can support certification when the
effect of unassigned probability mass can be propagated
conservatively to the target statistic. Such transfers nevertheless
require the relevant measurement assumptions, invariants, and bounds
to be revalidated for each new estimand and response interface.

\section{Limitations}
\label{sec:limitations}

The empirical scope of this study is deliberately narrow. The reported results
come from a single instruction-tuned model, with Track~M covering 18
English-language item pairs and Track~P covering the eight year-fixed S1/S2
pairs under a respondent-persona framing. They therefore do not support
cross-model, cross-linguistic, or cross-cultural generalization. Many items were
drawn from U.S.  political and social survey traditions, although the set also
included matched name-substitution controls and newly constructed judgment
pairs.  
The name-substitution contrast probes sensitivity to replacing the
named entities in historically surveyed items, but it cannot rule out
subtler forms of memorization, such as associations attached to the
substituted individuals or semantic prior knowledge of the survey
topics, and with only three matched pairs, it remains descriptive
rather than statistically powered for inferential claims.
We further examined only one measurement
scheme: binary-conditioned next-token log-probabilities obtained through
single-token forced-choice labels. The use of neutral A/B labels with an
explicit per-turn legend is itself part of this design, and other response
formats, model scales, base checkpoints, or estimands may produce different
patterns of label sensitivity, dispersion, and saturation. Accordingly, our
empirical conclusion is limited to the persistence of saturation under the
tested model--prompt--item--estimand configurations, rather than to
instruction-tuned LLMs or next-token measurement in general.

The inferential and interpretive scope is also limited. 
The health gates screen for pre-specified basic failure modes of the measurement
chain and are not an exhaustive verification of the implementation. The
preserved, hash-tracked execution records support post hoc auditing of failure
modes identified subsequently. 
Two Track~M items and three Track~P items remained $\Gamma$-indeterminate at the certification bounds. 
Our envelope procedure certifies noncontextuality under the adopted
CbD criterion when the certified upper bound on $\Gamma$ is nonpositive.
For the indeterminate items, however, failure to certify residual contextuality
should not be interpreted as evidence that it is absent. Human joint-probability
benchmarks were unavailable for direct calibration, so comparisons with
historical survey findings remain qualitative. More fundamentally, QQ
satisfaction establishes compatibility with multiple mechanism classes rather
than identifying a unique generating mechanism (Proposition~\ref{prop:symrep}).
Conversely, a QQ violation shows that the measured binary-conditioned joint
distributions are incompatible with the standard projective question-order
model; it does not identify the model's internal mechanism or exclude more
general quantum-like instrument models \citep{lebedev:2023}. Finally, the claim
that no prior study has tested the QQ equality on LLMs under a committed,
multi-turn sequential protocol is limited to the literature identified in
searches completed through July 2026.

\section{Conclusion}

We developed the QQ equality---an a priori constraint of the standard
projective quantum question-order model that has been observed across
human survey data---into an audit criterion for sequential binary
judgments of autoregressive LLMs. By characterizing
robust QQ satisfaction in several mechanism families, establishing
closure under order-matched mixing, and translating the rank-2
CbD criterion into the audit inequality $|\qQQ| \le \OSS$, 
we clarified what inferential claims QQ
satisfaction and violation can and cannot support. 

In the empirical pilot, all pre-specified health gates passed: the
canonical labels remained single tokens in the rendered contexts,
unassigned probability mass stayed negligible, and the logprob
readout agreed with trajectory-level sampling in the pre-specified
spot checks. Nevertheless, the binary-conditioned distributions were
saturated for 17 of 18 Track~M items and 7 of 8 Track~P items.
Counterbalanced mappings showed that neutral-label assignment could
materially change the measured $\qQQ$ and $\OSS$, changing the
mapping-specific QQ verdict for four items, and item s1-07 separated
order sensitivity, QQ satisfaction, and residual-contextuality
certification within a single item. The pooled audit produced both QQ
satisfaction and violation, but no item had a certified
$\Gamma_{\mathrm{lower}}>0$; many violations instead lay near the
boundary $|\qQQ|=\OSS$ and were therefore compatible with
noncontextual couplings that allow between-order marginal changes,
with several items remaining indeterminate at the certification
bounds. These findings do not establish whether the model is or is
not ``quantum-like.'' Rather, they show that, under the tested
model--prompt--item--estimand conditions, a forced-binary next-token
estimand can remain too nearly deterministic to support fine-grained
mechanism discrimination even when it passes every pre-specified
check.

Methodologically, the pilot shows that distribution-level audits of
LLMs should assess the adequacy of the response distribution before
interpreting higher-level structural criteria. We therefore recommend
a pre-specified saturation diagnostic as a standard health check
whenever next-token probabilities are treated as survey-response
distributions, accompanied by label counterbalancing and explicit
tracking of unassigned probability mass as supporting controls.
Future work should test whether this measurement caveat persists across broader
model families and measurement designs, with mixture-based designs
reporting component-level diagnostics alongside aggregate verdicts.

\section*{Code and Data Availability} \label{sec:availability}

All experiment code, item sets, the validated mathematical core with its
assertion suite, and the preserved, hash-tracked audit records (JSON) and run reports
for every execution reported here, including the retrospective batch-one
G3 re-validation records, are available at \codeurl\
(tag \texttt{arxiv-v2}, commit \texttt{fc32ccbd4b4777e45927f35921d0a96161086c3e}).

\section*{Acknowledgments}

This work was supported by the National Research Foundation of Korea grant funded by the Korea government (Ministry of Science and ICT), grant number RS-2026-25477171.

\clearpage
\appendix

\section{Proofs}\label{app:proofs}

Throughout, answers are denoted $\alpha,\beta\in\{y,n\}$, and we write
$u=p_A(y)$, $v=p_B(y)$. Recall the mismatch transition rates
$a=K_{AB}(n|y)$, $b=K_{AB}(y|n)$, $c=K_{BA}(n|y)$, $d=K_{BA}(y|n)$.
The disagreement probabilities take the form
\begin{equation}\label{eq:disagree}
M_{AB} \;=\; u\,a + (1-u)\,b,
\qquad
M_{BA} \;=\; v\,c + (1-v)\,d,
\end{equation}
since in order $AB$ a disagreement occurs when the first answer is $y$
(probability $u$) and the second flips (rate $a$), or the first is $n$
(probability $1-u$) and the second flips (rate $b$); symmetrically for
$BA$. 
The mechanism-family calculations below begin from
\eqref{eq:disagree}; Lemma~\ref{lem:transym},
Proposition~\ref{prop:mixing}, and Corollary~\ref{cor:cbd} use
additional structure noted in each proof.

\begin{proof}[Proof of Theorem~\ref{thm:mi}]
From \eqref{eq:disagree},
\begin{equation}
\qQQ = M_{AB}-M_{BA}
     = u(a-b) - v(c-d) + (b-d),
\end{equation}
which is affine in $(u,v)$ with coefficients independent of $(u,v)$ by
hypothesis. If $\qQQ$ vanishes on a set containing an open subset of
$(0,1)^2$, the affine function vanishes identically there, so its
coefficients vanish: $a=b$, $c=d$, and $b=d$; hence $a=b=c=d$.
Conversely, if $a=b=c=d$ then $M_{AB}=a=M_{BA}$ for all $(u,v)$, so
$\qQQ\equiv 0$.
\end{proof}

\begin{lemma}[Projective transition symmetry]\label{lem:transym}
For binary rank-1 projective sequential measurements on a qubit, at
any single item the mismatch transition rates coincide:
\begin{equation}
a=b=c=d=|\langle A_y|B_n\rangle|^2 ,
\end{equation}
whatever projectors are chosen for that item.
\end{lemma}
\begin{proof}
The flip rates are transition probabilities between eigenstates:
$K_{AB}(n|y)=|\langle B_n|A_y\rangle|^2$ and
$K_{AB}(y|n)=|\langle B_y|A_n\rangle|^2$. In dimension~2,
orthocomplementarity gives
$|\langle B_y|A_n\rangle|^2=|\langle B_n|A_y\rangle|^2$, and the same
quantities appear in order $BA$ since the modulus is symmetric in its
arguments; hence all four rates equal $|\langle A_y|B_n\rangle|^2$.
\end{proof}

\begin{remark}\label{rem:fixedpair}
If one measurement pair is fixed across items (item variation acting
on the state only), then by Lemma~\ref{lem:transym} the rates
$(a,b,c,d)$ are constants of the mechanism, and the model lies in the
marginal-independent class of Theorem~\ref{thm:mi}. If projectors
vary across items, the family-level fixed-kernel hypothesis of
Theorem~\ref{thm:mi} fails, but the item-level symmetry of
Lemma~\ref{lem:transym} remains.
\end{remark}

\begin{proof}[Proof of Proposition~\ref{prop:symrep}]
With $K_{AB}(\beta\,|\,\alpha)=r\,\mathbf{1}[\beta=\alpha]
+(1-r)\,p_B(\beta)$, the flip rates are
$a=(1-r)\,p_B(n)=(1-r)(1-v)$ and $b=(1-r)\,v$, so by
\eqref{eq:disagree}
\begin{equation}
M_{AB}=(1-r)\bigl[u(1-v)+(1-u)v\bigr].
\end{equation}
The same computation for
$K_{BA}(\alpha\,|\,\beta)=r\,\mathbf{1}[\alpha=\beta]
+(1-r)\,p_A(\alpha)$ gives
$M_{BA}=(1-r)\bigl[v(1-u)+(1-v)u\bigr]=M_{AB}$. Hence $\qQQ\equiv 0$
for every $(u,v)$ and every $r\in[0,1)$.

For the order effect, the second-position marginal in order $AB$ is
\begin{equation}
p_{AB}(B_y)=u\bigl[r+(1-r)v\bigr]+(1-u)(1-r)v
           = r\,u+(1-r)\,v,
\end{equation}
while $p_{BA}(B_y)=v$; by the sign convention of
Section~\ref{ss:setup}, $OE_B=p_{AB}(B_y)-p_{BA}(B_y)=r\,(u-v)$.
Finally, $a-b=(1-r)(1-2v)\neq 0$ whenever $0\le r<1$ and $v\neq 1/2$.
\end{proof}

\begin{remark}
By Lemma~\ref{lem:transym}, every 2D rank-1 projective sequential
kernel has $a=b$ at any single item, whatever projectors are chosen
for that item; hence for $r<1$ and at items with $v\neq 1/2$ the repetition kernel
is not realizable by any such projective kernel --- an item-level
exclusion that does not require the fixed-pair hypothesis of
Theorem~\ref{thm:mi}.
\end{remark}

\begin{proof}[Proof of Proposition~\ref{prop:mixture}]
With retention rates depending on polarity and order,
$K_{AB}(n|y)=(1-r_{AB}[y])\,p_B(n)$ and
$K_{AB}(y|n)=(1-r_{AB}[n])\,p_B(y)$, so
\begin{align}
M_{AB} &= (1-r_{AB}[y])\,u(1-v) + (1-r_{AB}[n])\,(1-u)v,\\
M_{BA} &= (1-r_{BA}[y])\,v(1-u) + (1-r_{BA}[n])\,(1-v)u,
\end{align}
and subtracting,
\begin{equation}\label{eq:mixq}
\qQQ = u(1-v)\,\bigl(r_{BA}[n]-r_{AB}[y]\bigr)
     + (1-u)v\,\bigl(r_{BA}[y]-r_{AB}[n]\bigr).
\end{equation}
The two monomials $u(1-v)$ and $(1-u)v$ are linearly independent as
functions on any open subset of $(0,1)^2$, so \eqref{eq:mixq} vanishes
on such a set iff both coefficients vanish:
$r_{AB}[y]=r_{BA}[n]$ and $r_{AB}[n]=r_{BA}[y]$.
The ``iff'' is relative to the robust (open-domain) notion of
Definition~\ref{def:mechanism}. On a boundary segment or a restricted
item family the two coefficients of these monomials need not be separately identified,
so \eqref{eq:mixq} may vanish without cross-polarity, cross-order symmetry.
The two special cases follow by substitution: order-invariant retention
($r_{AB}[\cdot]=r_{BA}[\cdot]=r[\cdot]$) gives
$\qQQ=(r[n]-r[y])(u-v)$, and polarity-invariant retention
($r_{XY}[y]=r_{XY}[n]=r_{XY}$) gives
$\qQQ=(r_{BA}-r_{AB})\,S$ with $S=u(1-v)+(1-u)v$.
\end{proof}

\begin{proof}[Proof of Proposition~\ref{prop:mixing}]
$\qQQ$ is a difference of two probabilities, each linear in the
corresponding order-conditioned joint; hence $\qQQ$ is a linear
functional of the pair
$J=(p_{AB},p_{BA})$. For components
$J^{(1)},\dots,J^{(K)}$ with $\qQQ(J^{(k)})=0$ and weights $w_k\ge 0$,
$\sum_k w_k=1$, applied identically to both coordinates (possibly
depending on $\theta$ but not on the order),
\begin{equation}
\qQQ\Bigl(\textstyle\sum_k w_k J^{(k)}\Bigr)
= \sum_k w_k\,\qQQ\bigl(J^{(k)}\bigr) = 0 .
\end{equation}
If instead the weights may depend on the order, the mixture value is
$\sum_k w_k^{AB} M^{(k)}_{AB} - \sum_k w_k^{BA} M^{(k)}_{BA}$, which
need not vanish even when each component satisfies QQ: two components
with different, order-invariant disagreement probabilities
$M^{(1)}\neq M^{(2)}$, selected with $w^{AB}=(1,0)$ and
$w^{BA}=(0,1)$, give $\qQQ=M^{(1)}-M^{(2)}\neq 0$. Hence
order-dependent selection is a possible, though not necessary,
violation channel.
\end{proof}

\begin{proof}[Proof of Corollary~\ref{cor:cbd}]
Code answers as $\pm 1$ ($y\mapsto +1$, $n\mapsto -1$). Following
Contextuality-by-Default, the two orders are two contexts, and the
variables measured in them are distinct random variables
$A^{AB},B^{AB}$ (jointly distributed) and $A^{BA},B^{BA}$ (jointly
distributed); no joint distribution across contexts is presupposed.
For binary $\pm1$ variables the within-context product moment equals
agreement minus disagreement probability,
$E[A^{c}B^{c}]=1-2M_{c}$ for $c\in\{AB,BA\}$, so
\begin{equation}
\bigl|E[A^{AB}B^{AB}]-E[A^{BA}B^{BA}]\bigr|
 = 2\,|M_{BA}-M_{AB}| = 2\,|\qQQ|.
\end{equation}
Similarly, a $\pm1$ mean is affine in the ``yes'' marginal,
$E[X^{c}] = 2p_{c}(X_y)-1$, so
\begin{equation}
\bigl|E[A^{AB}]-E[A^{BA}]\bigr| = 2\,|OE_A|,
\qquad
\bigl|E[B^{AB}]-E[B^{BA}]\bigr| = 2\,|OE_B|,
\end{equation}
by the sign conventions of Section~\ref{ss:setup}. The
Kujala--Dzhafarov noncontextuality criterion for rank-2 cyclic systems
with binary variables \citep{kujala:2016}, specialized to the
present two-context system in the notation above, states that a maximally
noncontextual coupling exists iff the between-context change of the
product moment does not exceed the total between-context change of the
marginals:
\begin{equation}
\bigl|E[A^{AB}B^{AB}]-E[A^{BA}B^{BA}]\bigr|
\;\le\;
\bigl|E[A^{AB}]-E[A^{BA}]\bigr|
+\bigl|E[B^{AB}]-E[B^{BA}]\bigr|.
\end{equation}
Substituting the three identities and dividing by $2$ yields
$|\qQQ|\le |OE_A|+|OE_B| = \OSS$. The stated consequences follow:
$\qQQ=0$ satisfies the inequality for any marginals, and
$|\qQQ|>\OSS$ is exactly its failure.
\end{proof}

\section{Prompts and Item Materials}
\label{app:materials}

This appendix documents the prompting conditions, label
counterbalancing, and the frozen 18-pair item set used in the pilot.
The complete machine-readable materials are provided in
Appendix~\ref{app:machine-readable}. The item file was frozen before
execution and was not expanded after observing the pilot results.

\subsection{System prompts}
\label{app:system-prompts}

\paragraph{Track M: direct-evaluation framing.}
Track M presented the questions with a neutral instruction and no
persona framing. For the S1 and S2 items, whose semantic response
options were mapped to the neutral labels \texttt{A} and \texttt{B},
the system prompt (template T-1, reproduced verbatim from the audit
record) was:

\begin{quote}
\small\ttfamily
You are asked a series of judgment questions. For each question,
answer with exactly one letter: "A" or "B". Do not explain.
\end{quote}

For the S3 items, which used the canonical labels \texttt{Yes} and
\texttt{No}:

\begin{quote}
\small\ttfamily
You are asked a series of judgment questions. For each question,
answer with exactly one word: "Yes" or "No". Do not explain.
\end{quote}

\paragraph{Track P: survey-respondent framing.}
The single redesign run R1 used the following system-prompt template,
with the item-specific reference year substituted before rendering:

\begin{quote}
\small\ttfamily
Answer each question as an average adult American survey respondent in
\{reference\_year\} would have answered. For each question, answer with
exactly one letter: A or B. Do not explain.
\end{quote}

Track P was applied only to the eight year-fixed S1/S2 pairs
\texttt{s1-01}, \texttt{s1-02}, \texttt{s1-04}, \texttt{s1-06},
\texttt{s1-11}, \texttt{s1-12}, \texttt{s2-01}, and
\texttt{s2-02}. Their substituted reference years were 1997, 1995,
1996, 1979, 2008, 2012, 1997, and 1995, respectively.

\paragraph{Historical headers in Track M.}
Two Track M items included an additional historical header. The header
was inserted once, before the first question in the sequential
interaction, and was not repeated before the second question.

For \texttt{s1-05}:

\begin{quote}
\small\ttfamily
The following question is from a survey conducted during the Cold War
era. Answer the question as written, based on your own assessment of
that period.
\end{quote}

For \texttt{s1-11}:

\begin{quote}
\small\ttfamily
The following question is from a survey conducted in 2008. Answer the
question as written, based on your own assessment of that period.
\end{quote}

No historical header was added in Track P. In particular, the
year-conditioned Track P system prompt replaced rather than
supplemented the Track M historical framing.

\subsection{Sequential presentation and label counterbalancing}
\label{app:presentation}

Each pair was presented in both orders, $A\!\rightarrow\!B$ and
$B\!\rightarrow\!A$. The first response was committed to the
conversation history as an assistant message whose content was exactly
the accepted canonical label token, before the second question was
presented. For each first-answer branch, the history therefore
contained the exact canonical response selected for the first question.

For S1 and S2, the two semantic response options were evaluated under
two counterbalanced mappings:

\begin{align*}
\text{map-1:}\quad&
\text{option 1}\mapsto \texttt{A},
\qquad
\text{option 2}\mapsto \texttt{B},\\
\text{map-2:}\quad&
\text{option 1}\mapsto \texttt{B},
\qquad
\text{option 2}\mapsto \texttt{A}.
\end{align*}

A mapping was held fixed across both presentation orders of a given
condition. Before pooling, all token-level probabilities and joint
distributions were transformed back to the original semantic option
coordinates. The two mappings were then pooled with equal,
pre-specified weights. Per-mapping results were retained separately as
label-effect diagnostics.

The active label mapping was displayed explicitly to the model in every
user turn. Each S1/S2 turn began with a two-line plain-text legend
assigning the semantic options to the neutral labels under the active
mapping, followed by the question. For example, for item \texttt{s1-04}
the legend read

\begin{quote}
\small\ttfamily
A = A few\\
B = Many
\end{quote}

\noindent under map-1, and

\begin{quote}
\small\ttfamily
A = Many\\
B = A few
\end{quote}

\noindent under map-2. Thus, the model was not required to infer the
meaning of the neutral labels from context.

S3 used the canonical labels \texttt{Yes} and \texttt{No} directly and
was analyzed as a separate stratum from the A/B-mapped S1/S2 items.

For concreteness, the following is the complete role-level message sequence for
item \texttt{s1-04} under map-1, order $A\!\rightarrow\!B$, on the
branch in which the first answer \texttt{A} was committed. 
Messages are shown at the role level: the square-bracketed role
names are display labels, not literal input strings, and the exact
serialization, including special tokens, is fixed by the recorded
chat template and its hash.

\begin{quote}
\small\ttfamily
[system]\quad You are asked a series of judgment questions. For each
question, answer with exactly one letter: "A" or "B". Do not
explain.\\[4pt]
[user]\;\quad A = A few\\
\phantom{[user]\quad} B = Many\\
\phantom{[user]\quad} Do you think that only a few or many white
people dislike black people?\\[4pt]
[assistant]\quad A\\[4pt]
[user]\;\quad  A = A few\\
\phantom{[user]\quad} B = Many\\
\phantom{[user]\quad} Do you think that only a few or many black
people dislike white people?
\end{quote}

\noindent
The first-position distribution is read at the answer position of the
first user turn (before the assistant turn is appended); the
branch-conditional distribution $K_{AB}(\,\cdot\,|\,\text{A few})$ is
read at the answer position following the second user turn. The
sibling branch commits \texttt{B} in the assistant turn instead, the
reverse order swaps the two questions, and map-2 swaps the legend
lines and the token used to commit the same semantic branch; 
all other elements are unchanged.

\subsection{Frozen item manifest}
\label{app:item-manifest}

Table~\ref{tab:item-manifest} summarizes the 18 frozen question pairs.
A dash indicates that no fixed reference year, matched S1 item,
vignette, or historical header was used. The column ``P'' identifies
the eight items included in the Track P redesign run. The full
verbatim question wording is contained in the archived JSON file.

\begingroup
\small
\setlength{\tabcolsep}{3.5pt}
\renewcommand{\arraystretch}{1.12}

\begin{longtable}{
    >{\raggedright\arraybackslash}p{0.09\textwidth}
    >{\raggedright\arraybackslash}p{0.06\textwidth}
    >{\raggedright\arraybackslash}p{0.16\textwidth}
    >{\raggedright\arraybackslash}p{0.24\textwidth}
    >{\centering\arraybackslash}p{0.08\textwidth}
    >{\centering\arraybackslash}p{0.06\textwidth}
    >{\centering\arraybackslash}p{0.07\textwidth}
    >{\raggedright\arraybackslash}p{0.10\textwidth}
}
\caption{Manifest of the frozen 18-pair pilot item set. ``Sens.''
indicates the pre-specified sensitive-item flag; ``P'' indicates
inclusion in the Track P redesign run.}
\label{tab:item-manifest}\\

\toprule
ID & Set & Pair type & Semantic options & Year & Sens. & P &
Additional material\\
\midrule
\endfirsthead

\multicolumn{8}{l}{\small\itshape
Table~\ref{tab:item-manifest} continued.}\\
\toprule
ID & Set & Pair type & Semantic options & Year & Sens. & P &
Additional material\\
\midrule
\endhead

\midrule
\multicolumn{8}{r}{\small\itshape Continued on next page.}\\
\endfoot

\bottomrule
\endlastfoot

\texttt{s1-01} & S1 & consistency
& Yes / No & 1997 & No & Yes & --\\

\texttt{s1-02} & S1 & consistency
& Yes / No & 1995 & No & Yes & --\\

\texttt{s1-04} & S1 & contrast
& A few / Many & 1996 & Yes & Yes & --\\

\texttt{s1-05} & S1 & reciprocity
& Yes / No & -- & No & No & Cold War header\\

\texttt{s1-06} & S1 & context contrast
& Yes / No & 1979 & Yes & Yes & --\\

\texttt{s1-07} & S1 & reciprocity
& Yes / No & -- & No & No & --\\

\texttt{s1-09} & S1 & consistency
& Approve / Disapprove & -- & No & No & --\\

\texttt{s1-10} & S1 & consistency
& Approve / Disapprove & -- & No & No & --\\

\texttt{s1-11} & S1 & reciprocity
& Too personally critical / Not too personally critical
& 2008 & No & Yes & 2008 header\\

\texttt{s1-12} & S1 & consistency
& Favor / Oppose & 2012 & Yes & Yes & --\\

\texttt{s2-01} & S2 & consistency
& Yes / No & 1997 & No & Yes & matched to \texttt{s1-01}\\

\texttt{s2-02} & S2 & consistency
& Yes / No & 1995 & No & Yes
& matched to \texttt{s1-02}; monitored\\

\texttt{s2-09} & S2 & consistency
& Approve / Disapprove & -- & No & No
& matched to \texttt{s1-09}\\

\texttt{s3-01} & S3 & consistency
& Yes / No & -- & No & No & vignette\\

\texttt{s3-03} & S3 & contrast
& Yes / No & -- & No & No & vignette\\

\texttt{s3-08} & S3 & contrast
& Yes / No & -- & No & No & vignette\\

\texttt{s3-09} & S3 & consistency
& Yes / No & -- & No & No & vignette\\

\texttt{s3-11} & S3 & mutual obligation
& Yes / No & -- & No & No & --\\

\end{longtable}
\endgroup

\subsection{Representative verbatim items}
\label{app:item-examples}

The following examples illustrate the three item strata and the
different response formats. Question wording and semantic options are
reproduced verbatim from the frozen item file.

\subsubsection*{Example 1: original survey item with non-Yes/No options
(\texttt{s1-12})}

\paragraph{Question A.}
\begin{quote}
Do you generally favor or oppose affirmative action programs for racial minorities?
\end{quote}

\paragraph{Question B.}
\begin{quote}
Do you generally favor or oppose affirmative action programs for women?
\end{quote}

\noindent
Semantic options: \emph{Favor} and \emph{Oppose}. These options were
mapped to \texttt{A}/\texttt{B} under both counterbalanced mappings.
The item was marked sensitive and used reference year 2012 in Track P.

\subsubsection*{Example 2: matched name-substitution item
(\texttt{s2-01})}

\paragraph{Question A.}
\begin{quote}
Do you generally think Tony Blair is honest and trustworthy?
\end{quote}

\paragraph{Question B.}
\begin{quote}
Do you generally think Gordon Brown is honest and trustworthy?
\end{quote}

\noindent
Semantic options: \emph{Yes} and \emph{No}. The item substitutes
contemporaneous UK figures for the U.S. figures of \texttt{s1-01},
preserving the era and the paired-politician structure, and used
reference year 1997 in Track P.

\subsubsection*{Example 3: novel vignette item (\texttt{s3-01})}

\paragraph{Vignette.}
\begin{quote}
Two candidates applied for the same software engineering role.
Candidate K has strong qualifications but a two-year employment gap.
Candidate L has equivalent qualifications but changed jobs five times
in six years.
\end{quote}

\paragraph{Question A.}
\begin{quote}
Should Candidate K advance to the interview stage?
\end{quote}

\paragraph{Question B.}
\begin{quote}
Should Candidate L advance to the interview stage?
\end{quote}

\noindent
The canonical responses for this stratum were \texttt{Yes} and
\texttt{No}; no A/B remapping was applied.

\subsection{Complete machine-readable materials}
\label{app:machine-readable}

The verbatim wording, semantic response options, metadata, vignettes,
historical headers, matching relations, and monitoring flags for all
18 pairs are provided in:

\begin{quote}
\texttt{items/pilot\_v1\_2.json}
\end{quote}

The version archived at the initial arXiv submission is identified by
the repository commit \texttt{4563023}. The file is unchanged in the
current submission snapshot (tag \texttt{arxiv-v2},
Section~\ref{sec:availability}).

The preserved pilot execution records additionally store the rendered system
prompt, model and tokenizer revision fields, chat-template hash, canonical
token identifiers, mapping assignment, random seed, numerical precision 
and software-stack versions, and sampling-check metadata for each reported pilot run.

\end{document}